\documentclass[12pt, oneside, reqno]{amsart}
 \usepackage{graphicx}
\usepackage{amsmath,amsthm,amsfonts,amscd,amssymb,mathrsfs,cite}
\usepackage[initials]{amsrefs} 
\usepackage[all]{xypic}

\usepackage{amscd}   
\usepackage[all]{xy} 
\usepackage{booktabs} 
\usepackage{stmaryrd}
\usepackage{mathtools}
\usepackage{hyperref}
\usepackage{epstopdf}
\usepackage[utf8x]{inputenc} 
\usepackage{soul} 
\usepackage{nccmath} 
\usepackage{url}
\DeclareMathAlphabet{\mathscr}{OT1}{pzc}{m}{it} 

\DeclarePairedDelimiter\ceil{\lceil}{\rceil}

\makeatletter
\g@addto@macro\normalsize{%
  \setlength\abovedisplayskip{4pt}
  \setlength\belowdisplayskip{4pt}
  \setlength\abovedisplayshortskip{4pt}
  \setlength\belowdisplayshortskip{4pt}
}
\makeatother
\usepackage{hyperref}
  \hypersetup{colorlinks=true,citecolor=blue, urlcolor= cyan}

 \AtBeginDocument{%
    \def\MR#1{} 
 }

\usepackage{color}

\newtheorem{theorem}{Theorem}[section]

\newtheorem{prop}[theorem]{Proposition}

\newcommand{\SL}{\operatorname{SL}}
\newcommand{\GL}{\operatorname{GL}}

\newcommand{\PP}{\mathbb{P}}

\newcommand{\RR}{\mathbb{R}}

\newcommand{\CC}{\mathbb{C}}

\newcommand{\Seg}{\operatorname{Seg}}

\def\phi{ \varphi }

\def \H{\mathcal{H}}
\def \Hn{\mathcal{H}^{\otimes n}}

\theoremstyle{definition}

\newtheorem{examplex}{Example}
\newenvironment{example}
  {\pushQED{\qed}\examplex}
  {\popQED\endexamplex}

\theoremstyle{remark}
\newtheorem{remark}[theorem]{Remark}

\newcommand{\bra}[1]{{\left\langle{#1}\right\vert}}
\newcommand{\ket}[1]{{\left\vert{#1}\right\rangle}}
\newcommand{\braket}[2]{{\left\langle {#1}\vert{#2}\right\rangle}}

\begin{document}
\date{\today}

\author{Hamza Jaffali}\email{hamza.jaffali@utbm.fr}
\address{Femto-ST/UTBM, Universit\'e de Bourgogne Franche-Comt\'e, 90010, Belfort, France
}


\author{Luke Oeding}\email{oeding@auburn.edu}
\address{Department of Mathematics and Statistics,
Auburn University,
Auburn, AL, USA
}
\title{Learning Algebraic Models of Quantum Entanglement}

\begin{abstract}
We review supervised learning and deep neural network design for learning membership on algebraic varieties.
We demonstrate that these trained artificial neural networks can predict the entanglement type for quantum states. We give examples for detecting degenerate states, as well as border rank classification for up to 5 binary qubits and 3 qutrits (ternary qubits). 
\end{abstract}
\maketitle
\section{Introduction}
Recent efforts to unite Quantum Information, Quantum Computing, and Machine Learning have largely been centered on integrating quantum algorithms and quantum information processing into machine learning architectures \cite{Wiebe2012,Lloyd2014,Rebentrost2014,Schuld2014,Schuld2015,Wiebe2016,Biamonte2017,Sheng2017,Kerenidis2018,Wossnig2019}. Our approach is quite the opposite -- we leverage Machine Learning techniques to build classifiers to distinguish different types of quantum entanglement. Machine Learning has been used to address problems in quantum physics and quantum information, such as quantum state tomography \cite{Quek2018}, quantum error correction code \cite{Nautrup2018} and wave-function reconstruction \cite{Beach2018}.  Here we focus on quantum entanglement. While we were inspired by the approach of learning algebraic varieties in \cite{Breiding2018}, our methods differ in that we are not trying to find intrinsic defining equations of the algebraic models for entanglement types, but rather building a classifier that directly determines the entanglement class. Distinguishing entanglement types may be useful for quantum information processing and in improving and increasing the efficiency of quantum algorithms, quantum communication protocols, quantum cryptographic schemes, and quantum games. Our methods generalize to  cases where a classification of all entanglement types is not known, and cases where the number of different classes is not finite (see for instance \cite[Ch.10]{LandsbergTensorBook}, or \cite{Vinberg-Elasvili}).

We only focus on \emph{pure states} for representing quantum systems, which is sufficient for studying quantum computations and quantum algorithms. This is opposed to the noisy approach with density matrices and mixed states, which is used when one needs to account for the noise and the interaction with the environment \cite{Nielsen2011}. 

\subsection{Basic notions} A basic reference for tensors is \cite{LandsbergTensorBook}. 
The quantum state of a particle can be represented by a unit vector $\ket{\psi}$ in a Hilbert space $\H$ (typically $\H=\CC^d$ or $\RR^{d}$), with basis $\{ \ket{x} \mid x \in \llbracket 0, d-1 \rrbracket\}$ in decimal notation. The state of an $n$-qudit quantum system is represented by a unit vector $\ket{\psi}$  in a tensor product $\Hn$ of the state spaces for each particle, where for simplicity we assume all particles of the same type.  The tensor space $\Hn$ has basis $\ket{ij\ldots k}:= \ket i \otimes \ket j \otimes \dots \otimes \ket k$, where $\ket i, \ket j, \ldots, \ket k$ are basis elements of $\H$. The dual vector space $\H^*$ is represented by vectors $\bra{\phi}$, and we use the standard Hermitian inner product $ (\ket \psi , \ket \phi) \mapsto \braket{\phi}{\psi} =: \sum_{i} \overline{\phi_i} \psi_i \in \CC$, where $\ket \phi = \sum_i \phi_i \ket i$, and similarly for $\psi$. 

The study of quantum entanglement is often focused on orbits (equivalence classes) under the action of the SLOCC (Stochastic Local Operations and Classical Communication) group, which algebraists know as $G=\SL(\H)^{\times n}$, the cartesian product of special linear groups, \emph{i.e.} normal changes of coordinates in each mode.

\subsection{First examples: separating states by algebraic invariants} Consider the pair of binary state particles, called a 2-qubit system. Here there are only two different entanglement types up to the SLOCC action, represented by $\ket{00}$ and $\frac{1}{\sqrt{2}}\left(\ket{00} + \ket{11}\right)$. All other 2-qubit states can be moved to one of these by the action of the SLOCC group \cite{Miyake2002}. Determining on which orbit a given state $\ket{\phi}$ is (and hence its entanglement type), can be done by computing the $2\times 2$ matrix determinant: Set $\phi_{i,j}:=\braket{ij}{\phi}$. The value of $\det\left( \begin{smallmatrix} \phi_{00} & \phi_{01} \\ \phi_{10} & \phi_{11} \end{smallmatrix} \right)$ is either 0 or $\frac{1}{2}$ if $\ket{\phi}$ is respectively of type $\ket{00}$ or $\frac{1}{\sqrt{2}}\left(\ket{00} + \ket{11}\right)$ \cite{Miyake2002}. 
In particular, the non-general states live on a quadric hypersurface \cite{CarliniGrieveOeding}. An Artificial Neural Network classifier (see Section~\ref{sectionANN}) can also be trained to test membership on these two types of states. 

For 3-qubit systems, the rank (a numerical invariant) and the determinant (a polynomial invariant) generalize to the multilinear rank and the Cayley hyperdeterminant \cite{Miyake2002}, respectively. 
A given state $\ket \phi \in \H^{\otimes 3} = \CC^2\otimes\CC^2\otimes\CC^2$ has coordinates
\[
\phi_{ijk} = \braket{ijk}{\phi}. 
\]
The three 1-flattenings are as follows:
\[
F_{1}(\ket{\phi}) = (\phi_{ijk})_{i,jk}, \quad
F_{2}(\ket{\phi}) = (\phi_{ijk})_{j,ik}, \quad
F_{3}(\ket{\phi}) = (\phi_{ijk})_{k,ij}
.\]
The ranks of these flattenings comprise a vector called the \emph{multilinear rank}. 
No flattening has rank 0 since that only occurs if $\ket{\phi} =0$, but $\ket{\phi}$ is a unit vector. If the multilinear rank is $(1,1,1)$, then $\ket{\phi}$ is separable. If the multilinear rank is $(1,2,2)$ then $\ket{\phi}$ is a bi-separable state of the form $\frac{1}{\sqrt 2}\left(\ket{000} + \ket{011}\right)$ up to SLOCC. The other bi-separable states are permutations of this up to SLOCC.  Finally, if  the multilinear rank is $(2,2,2)$, then up to a SLOCC transformation, $\ket{\phi}$ is either the so called W-state $\frac{1}{\sqrt 3}\left(\ket{001} + \ket{010} + \ket{100}\right)$ or a general point $\frac{1}{\sqrt 2}\left(\ket{000} + \ket{111}\right)$ \cite{Miyake2002,Holweck2012}. These states are distinguished, respectively, by the vanishing or non-vanishing of the well-known SLOCC-invariant, the $2\times2 \times2$ hyperdeterminant:
\[\begin{split}
\Delta_{222}(x) =\; &x_{000}^{2}x_{111}^{2}
+x_{011}^{2}x_{100}^{2}
+x_{010}^{2}x_{101}^{2}
+x_{001}^{2}x_{110}^{2}
\\ &
-2x_{000}x_{011}x_{100}x_{111}
-2x_{000}x_{010}x_{101}x_{111}
-2x_{000}x_{001}x_{110}x_{111}
+4x_{000}x_{011}x_{101}x_{110}
\\&
-2x_{010}x_{011}x_{100}x_{101}
-2x_{001}x_{011}x_{100}x_{110}
-2x_{001}x_{010}x_{101}x_{110}
+4x_{001}x_{010}x_{100}x_{111}
.\end{split}
\]
These are all the possible types of states up to SLOCC for 3-qubit systems, and we have a complete algebraic description of these states as well. So, this is a good testing ground for other methods since we know how to test ground truth algebraically. In Section~\ref{sec:3qubit} we summarize the performance of a neural network for separating these states. 

In the 4-qubit case there is still a classification \cite{Verstraete2002}, and a complete set of invariants that can be used to classify entanglement types is known \cite{Holweck2014,Holweck2017}. In addition, algebraic invariants for all border ranks for 4-qubit systems are known classically (originally studied by Segre \cite{Segre1920}) and are given by the minors of 1- and 2- flattenings \cite[Ch.~7.2]{LandsbergTensorBook}. 
We note that the $2\times 2\times 2\times 2$ hyperdeterminant has degree 24, and is computable by Schlafli's method, \cite{GKZ}. However, in the 5-qubit case the $2^{\times 5}$ hyperdeterminant has degree 128, is not computable by Schlafli's method, \cite{WeyZel_Sing}, and is surely very complicated as indicated by the 4-qubit case \cite{HSYY_hyperdet}. However, in Section~\ref{dual_classifier} we show that neural networks can distinguish between singular and non-singular states even though algebraic invariants are likely to fail.

\subsection{Prior Work and Outline}
Machine Learning has been used to study entanglement detection, entanglement measurement, and entanglement classification. A common focus is the density matrices formalism, for which the asymptotic problem of deciding separability is a NP-hard problem. Neural networks have been used for estimating entanglement measures such as logarithmic negativity, R\'enyi entropy or concurrence in 2-qubits (pure and mixed) or many-body systems \cite{Gray2018,Berkovits2018,Govender2017}; for encoding several CHSH (witness) inequalities simultaneously in a network to detect entangled states \cite{Ma2018}; for computing the closest separable state in a complex valued network \cite{Che2018}; for recognizing the entanglement class of 3-qubit systems \cite{Behrman2011}; and for detecting entanglement in families of qubits and qutrits systems in the bipartite case \cite{Wisniewska2015}. A convex hull approximation method combined with decision tree and standard ensemble learning (bagging) algorithms was used in \cite{Lu2018} to classify separable and entangled states. The forest algorithm (also using decision trees) was used in \cite{Wang2017} to detect entanglement and was compared to quantum state tomography (up to five qubits). Principal Component Analysis (PCA) was used to determine the dimension and intrinsic defining equations of certain algebraic varieties \cite{Breiding2018}.  In the 2-qubit case neural networks, support vector machines, and decision trees were used to detect non-classical correlations such as entanglement, non-locality, and quantum steering \cite{Yang2018}. 

Indeed, Machine Learning can be a relevant tool for entanglement detection, at least in some limited cases. However, prior studies were mostly limited to the 2-qubit or bipartite case, because of the possibility of generating properly separable and entangled mixed states using the PPT criterion or Schmidt decomposition. It seems difficult to generalize the prior supervised approaches to higher dimensions or to systems with more than two particles. 

Instead, we focus on the problem of entanglement detection and classification by learning algebraic varieties  (Segre variety, dual variety, secant varieties) that characterize different entanglement classes for pure states. Our method can be generalized to higher dimensions and systems with several particles, bringing original tools for distinguishing non-equivalent entanglement classes for quantum systems for which we do not know the complete classification or we do not have exact algorithms (as proposed by \cite{Holweck2012,Holweck2014,Holweck2017}) to determine the entanglement type (as it is the case for 5-qubit systems for instance).

As noted in \cite{Breiding2018} there are many instances of high degree hypersurfaces (like the hypersurface of degenerate states in $\Hn$) that are easy to sample, but there is little hope to learn their equations from samples, so prior methods do not apply. Yet, in Section~\ref{dual_classifier} we give cases where an artificial neural network can be trained to determine membership on these high degree hypersurfaces, even without the inaccessible defining equation.

In Section~\ref{sectionDesign}, we investigate several architectures of feed-forward neural networks for learning membership on algebraic varieties, by using previous knowledge about the defining equations to design networks. In Section~\ref{sectionExperiments}, we study several examples by building classifiers with neural networks for  the cases of qubits and qutrits. Finally, in Section~\ref{sectionApplication}, we use these classifiers to distinguish quantum entanglement classes. 

\section{Network design for learning algebraic varieties}\label{sectionDesign}

\subsection{Basics of Artificial Neural Networks}\label{sectionANN}
For a complete introduction to Artificial Neural Networks and related concepts, see any of \cite{Lippmann1988,Barron1988,Khanna1990,Widrow1990,Vemuri1992,Wang1993,Hassoun1995,Haykin1998,anthony2009neural}. Here we give a brief overview.
Inspired by biological neural networks, Artificial Neural Networks (ANN) are computing systems whose goal is to reproduce the functionality and basic structure of the human brain \cite{McCulloch1990}.  
An artificial neural network is composed of several artificial neurons, regrouped using a specific architecture, which is designed to be trained to perform a specific task (such as classification and regression) without any explicit instructions or rules \cite{Haykin1998}.

In 1943, McCulloch and Pitts proposed the first model of an artificial neuron \cite{McCulloch1990}. An artificial neuron (Figure \ref{artificial_neuron}) is defined by inputs (eventually coming for other neurons), by weights (synaptic weights for each input), a weighted sum (computed from the two previous ones), a threshold (or bias), an activation function, and an output \cite{Lippmann1988,Barron1988,Hassoun1995,Haykin1998,anthony2009neural}.

\begin{figure}[!h]
\centerline{\includegraphics[width=.5\textwidth]{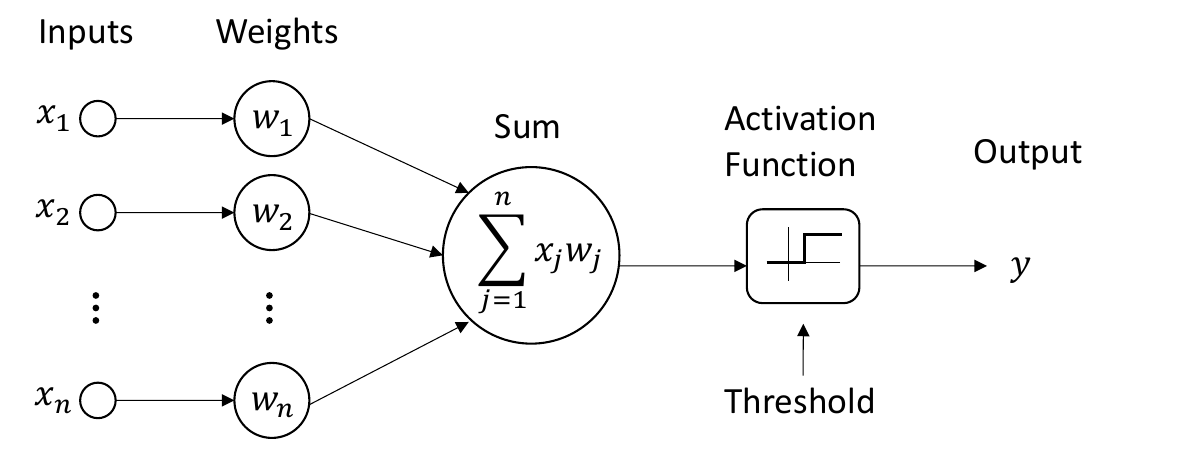}}
\caption{Illustration of an artificial neuron.}
\label{artificial_neuron}
\end{figure}

The output $y= g(\mathcal{U})$ of the neuron is equal to the activation function $g$ applied to the \emph{weighted sum} (with the threshold $\theta$ frequently added to the weighted sum with weight equal to 1) \cite{Cybenko1989}. If we denote by $x_1,\dots,x_n$ the inputs of the neuron, and by $w_i$ the weight associated with the $i$-th input, then the weighted sum, frequently denoted by $\mathcal{U}$, is defined as follows 
\begin{equation}
\mathcal{U} = w_1 x_1 + w_2 x_2 + \cdots + w_n x_n + \theta ~.
\end{equation}

We implemented neural networks in Python using the Keras library \cite{Chollet2015} with the \emph{Nadam} solver for minimizing the loss function (either binary or categorical cross-entropy) in the learning step \cite{Sutskever2013,Dozat2016}.
This setup allows flexible implementation of feed-forward neural networks and the choice of parameters (architecture, activation functions, loss function, etc.).

%
%

Geometrically, if one can find a hyperplane (codimension 1) that separates the data into two classes
 (such as the logical \texttt{OR} function), 
then the binary classification problem can be solved using a single artificial neuron \cite{Kung1992,Wang1993,Hassoun1995,Cirrincione2010}. 
%
On the other hand, we are interested in cases where there is not a separating hyperplane, (such as in the case of the logical \texttt{XOR} function \cite{Bishop1995,Hassoun1995,Borne2007}). So, we use more complex structures, such as the Multi-Layer Perceptron (MLP) model \cite{Barron1988,Haykin1998}. We used the feed-forward dense layer configuration.


Once activation and loss functions are fixed, the remaining parameters to configure are the number and size of the layers. The \emph{Universal Approximation Theorem} and related studies \cite{Cybenko1989,Hornik1991,Bishop1995,Hassoun1995,Csaji2001,Lu2017,Hanin2017}, show that ANNs can approximate any reasonable function. However, the question of providing the best network, in terms of accuracy and computational efficiency, for a given task or problem, is still open. The choice of the depth and the width of neural networks are often chosen by trial and error, considering the trade-off between performance and computational cost.


\subsection{Learning algebraic varieties}\label{learning_algebraic_varieties}
Here we present a model of ANN for learning membership on algebraic varieties modeling quantum entanglement. 
An algebraic variety is a geometric object defined as the zero locus of a set of homogeneous polynomials. 
In order to teach a machine how to recognize points on algebraic varieties, we must encode the polynomial defining equations of these objects (or an approximation of them) into the learning model. We would like to do this as efficiently as possible (using the least number of parameters) to avoid overfitting.

It is an instructive exercise to show that ANNs can be trained to recognize (determine membership) points on linear spaces (of any co-dimension) essentially by linear interpolation; see \cite{Raturi2018}, or any of \cite{Duda1973,Lippmann1988,Gibson1990,Khanna1990,Kung1992,Bishop1995,Haykin1998,anthony2009neural,Dreyfus2002,Borne2007}. To classify points on algebraic varieties one can also use ANNs as an alternative to polynomial interpolation. The so-called Polynomial Neural Networks (PNN) and were proposed in \cite{Oh2003}. See also any of  \cite{Specht1967,Ivakhnenko1971,Widrow1990,Kung1992,Shin1992,Shin1995,Bishop1995,Hassoun1995,Dreyfus2002}.

Let us collect some well-known results in algebraic geometry relevant to modeling polynomials with ANNs. For simplicity we focus on modeling a single polynomial. For several polynomials one may design a neural network using a copy of the one we describe here for each polynomial.  One way to determine a polynomial is to calculate the coefficients for every monomial. One might expect that this would not be the most efficient way to represent polynomials that may have additional structure, such as sparseness. A sum of powers representation has a chance to exploit hidden sparseness or other structure. 

What shape of ANN with power-function activations is necessary to model a polynomial equation of degree $d$ in $n$ variables? Such an ANN essentially accomplishes the task of writing the polynomial as a sum of powers of linear forms, and it is tightly linked (by apolarity) to interpolating polynomials. In this case a straightforward dimension count gives a good guess for the appropriate architecture. The Alexander-Hirschowitz (AH) Theorem (see \cite{Brambilla2008} for a modern treatment) tells us when the naive dimension count fails. In fact, the AH theorem states that a general\footnote{By ``general'' we mean avoiding a measure-zero set of possible counterexamples.} homogeneous polynomial $p$ of degree $d$ in $n$ variables can be expressed as the sum of $T=\ceil{\frac{1}{n}\binom{d+n-1}{d}}$ $d$-th powers of linear forms (except for quadratic forms 
 and a few other special cases), \emph{i.e.}
\[
p(x_1,x_2,\dots,x_n) = \sum_{j=1}^{T} \left( \sum_{i=1}^{n} a_{ij}x_i \right)^d .
\]

Dehomogenizing the AH result (by setting the last variable to 1, for example) one obtains a bound for 2-layer neural networks modeling affine hypersurfaces.
We implement $T=\ceil{\frac{1}{n}\binom{d+n-1}{d}}$ neurons with activation function $g : x \mapsto x^d$ in the first layer, and then combine all the outputs in a linear combination using a neuron in the second layer. Suppose we have $n$ inputs corresponding to the $n$ variables of the homogeneous polynomial $p$. If we denote by $w_{i,j}$ the weight associated with the $i$-th input $x_i$ and the $j$-th neuron in the first layer, and by $\theta_j$ the threshold of the $j$-th neuron, then the output $s_j$ of the neuron $j$ in the first layer is
\[
s_j = g(\mathcal{U}_j) = \left( \sum_{i=1}^{n} w_{i,j}x_i + \theta_j \right)^d .
\]

The threshold $\theta_j$  introduces an inhomogeneity which we can remove by adding an extra variable $x_{n+1}$, replacing $\theta_j$ with $\theta_jx_{n+1}$. Then the AH theorem for $n+1$ variables, with $x_{n+1}=1$ yields the bound for non-homogeneous outputs.  This idea also appears in \cite{Kileel2019}. 

The last step of the network design is classification: is the input point on (or outside) the variety defined by the equations modeled by the network? Note that after the two first layers, we should obtain a value $s$ (which is 0) when the input point is on the variety, and a different value if the point is not. So, this step is equivalent to recognizing a real number $s$ in an interval (whose size depends on the training data). Adding a single sigmoid neuron will not solve the classification problem because single neurons only solve inequalities, not equalities \cite{Gibson1990,Kung1992}. So, we add another layer before the output layer to recognize this specific value $s$, and then solve this binary classification problem at the output of the network.

The task of recognizing a real number in an interval can be performed using a single layer with four neurons with LeakyReLU activation functions. Thus, by adding such a layer after the first two layers (modeling the equations of the variety), and by adding a last layer (output layer) with only one neuron (with a sigmoid activation function), one can potentially learn any algebraic variety defined by a set of homogeneous polynomials (Figure~\ref{network_last_step}).

\begin{figure}[!h]
\centerline{\includegraphics[width=.6\textwidth]{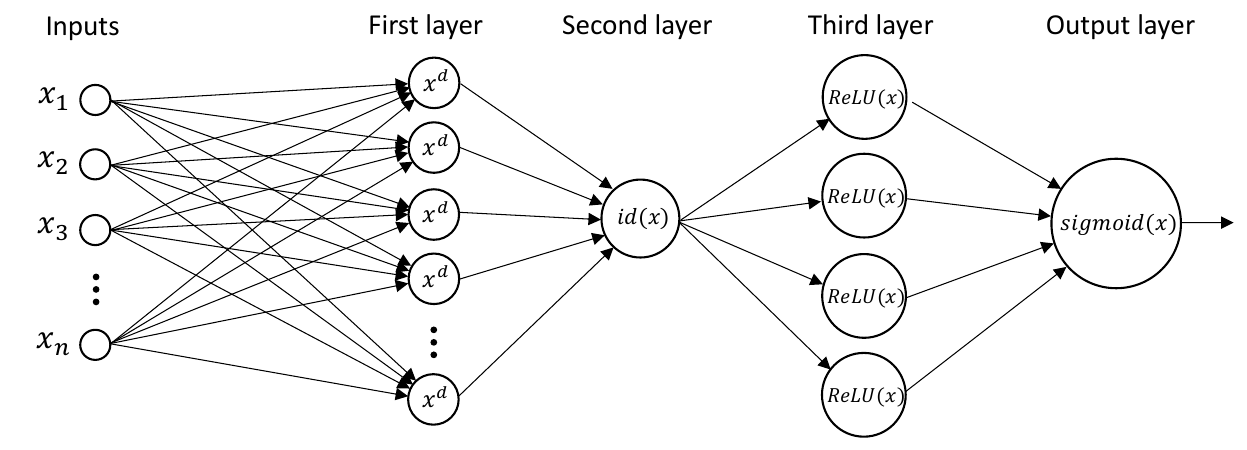}}
\caption{Representation of the network solving the binary classification problem related to algebraic variety's membership.}
\label{network_last_step}
\end{figure}

For practical implementations it is important to know if we can correctly train the network, and if the optimizer can find the right weights and thresholds. For homogeneous equations of low degree and few variables, the network is able to learn the correct weights, and one can recover the equation from the weights of the network. However, when the degree and the number of variables increase, the number of neurons needed in the first layer increases very quickly and it becomes harder to converge to a set of weights and thresholds that model the desired equation. To overcome this, in Subsection~\ref{hybrid_section} we slightly modify the network structure to reduce the number of weights and parameters. For more information on the expressive power of deep neural networks, see \cite{Kileel2019}. 

\begin{remark}
The Universal Approximation Theorem leads one to ask how training a neural network is different from polynomial fitting. Indeed, neural networks are know to be essentially polynomial regression models, with the effective degree of the polynomial growing at each hidden layer \cite{Cheng2018}. 
The correspondence between the values of the function to be interpolated, the basis functions and the basic points on one hand, and the weights, the activations functions and the thresholds (bias) on the other hand can be made explicit \cite{Li2003}. However, in practical applications, data is often noisy and incomplete, and polynomial interpolation is generally subject to overfitting \cite{Witten2016}, while neural networks are able to perform when there are noisy or incomplete data, and have the ability to generalize from the input data \cite{Ong2006}. This dilemma between fitting the training dataset and being able to generalize the models to non-encountered data is known as the bias-variance trade-off \cite{Belkin2018}. Raturi explains in \cite{Raturi2018} that solving the same problem as polynomial interpolation requires much less computational time and resources when using neural networks. 
It is also known that neural networks can interpolate and model a function via sigmoidal functions to approximate $n$  samples in any dimension, with arbitrary precision and without training \cite{Llanas2006}.
\end{remark}

\subsection{Hybrid networks}\label{hybrid_section}


Here we introduce a hybrid network architecture for classifying membership on a parametrized algebraic variety and discuss its advantages. Training an ANN, which is an optimization process, doesn't always reach the set of weights that minimize the loss function. This is due to the existence of local minima, related to the utilization of the non-convex $x \mapsto x^d$ activation functions.

The second layer (Figure~\ref{network_last_step}) contains only one neuron (with the identity activation function), introduced to take the linear combination of $d$-th power of weighted sums of input variables. One would like to remove this layer, and directly link outputs of the first layer with the third layer (containing only neurons with LeakyReLU activation function). We call these \emph{hybrid networks}, because they combine layers with both $x \mapsto x^d$ and LeakyReLU activation functions.

Moreover, the geometric interpretation of the network is now different since every neuron in the second layer (LeakyReLU neurons) take as input a different linear combination of all $d$-th power forms. The second layer combines different homogeneous polynomials that are not necessarily equal to the homogeneous polynomial defining the algebraic variety, but they can be used to approximate this last as a set of inequalities. The third and last layer of the network is the output layer, which will depend on the classification problem one wants to solve (binary classification, several classes, etc.). The architecture of the hybrid network is summarized in Figure~\ref{network_hybrid}.

\begin{figure}[!h]
\centerline{\includegraphics[width=.6\textwidth]{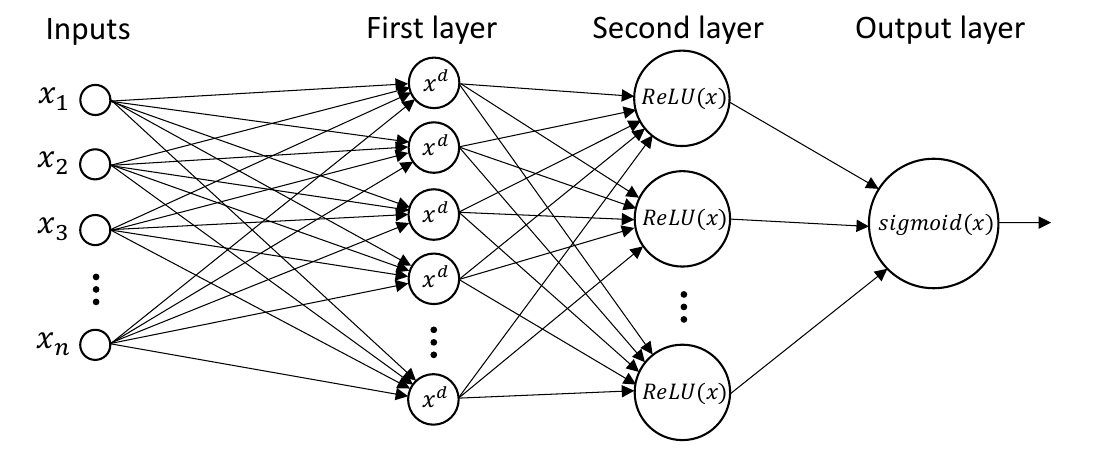}}
\caption{Representation of a hybrid network for learning a homogeneous polynomial equation of degree $d$ in $n$ variables and coefficients in $\RR$.}
\label{network_hybrid}
\end{figure}

This type of network showed better results and quicker learning than the previous version of the network. The AH theorem still can be used to give an idea about the number of neurons in the first layer, and sometimes extra neurons should be added or removed to boost the performance of the network during the learning phase. By experience, the number of neurons in the second layer can be chosen quite small (between 4 and 10 for binary classification problems) and will depend on the number of neurons in the output layer (it might not be smaller than the number of neurons in the output layer).

\section{Experiments}\label{sectionExperiments}

Consider a vector space $V$ and a group $G \subset \GL(V)$. The set of all such pairs $(V,G)$ for which $G$ acts on $V$ with finitely many orbits has been classified by Kac \cite{Kac80, Kac85}. Since then Vinberg and others have classified all the orbits in many of these cases. There are special cases where $G$ acts on $V$ with infinitely many orbits, yet those orbits may still be represented using finitely many parameters, the \emph{tame} case. A fantastic example of such is \cite{Vinberg-Elasvili}, which gave a classification of the orbits of 9 dimensional trivectors utilizing a connection to the Lie algebra $\mathfrak{e}_8$. This classification also implies a classification for 3 qutrit systems among others. After orbit classification, one desires effective methods to determine orbit membership. Separating orbits is difficult in general, yet this mathematical problem lies at the core of geometric approaches to Valliant's version of P versus NP \cite{LandsbergCambridgeBook}.

One approach is to use Invariant Theory to build a set of invariants and covariants that characterize each orbit, or each family represented by parameters \cite{Holweck2012,Luque2003}. In the four-qubit case, an infinite tame case, Verstraete \textit{et al.} gave a list of 9 normal forms (depending on parameters) that give a parametrization of all SLOCC orbits \cite{Verstraete2002,Chterental2006}. An algorithm  to determine orbit membership for the four-qubit case was proposed in \cite{Holweck2014,Holweck2017}. In general, the complexity of these invariants grows very rapidly with the number of particles in the system. Already in the five-qubit case, a complete description of the algebra of SLOCC-invariant polynomials is out of reach of any computer system \cite{Luque2006}.

Therefore, it is worth considering other approaches. In this section we give examples that serve as proof of the concept that artificial neural networks can be trained to give efficient and effective classifiers for quantum states.


\subsection{Detecting separable states}\label{segre_classifier} 
Here we focus the real case, and assume $\H = \RR^{d}$.
A state $\ket{\phi} \in \Hn$ is \emph{separable} if it can be written as a tensor product of pure states as follows 
\[
\ket{\phi} = \ket{\phi_1} \otimes \cdots \otimes \ket{\phi_n}
,\] 
for $\ket{\phi_i} \in \H$ for all $i$ \cite{Nielsen2011}. Algebraic geometers call the projective variety of separable states the Segre variety, which is parametrized as:
\[\begin{matrix}
\Seg\colon & \PP \H \times \dots \times \PP \H &\to& \PP \Hn \\
&([x_1],[x_2],\ldots, [x_n]) & \mapsto& [x_1 \otimes x_2 \otimes \ldots \otimes x_n]
\end{matrix},
\]
where the square brackets $[\cdot]$ denote the equivalence class by rescaling, which makes physical sense to consider since we assume states are unit vectors \cite{Holweck2012}.

The following is straightforward, essentially known to Segre in the early 1900's. 
\begin{prop}\label{prop:separable}
A state $\ket{\phi} \in \Hn$ is separable if and only if either of the following equivalent conditions hold. 
\begin{enumerate}
\item $\ket{\phi}$ is in the SLOCC orbit of $\ket{00\cdots0}$. 
\item All $n$ of the 1-flattenings $F_i(\ket{\phi}) \colon \H^{\otimes (n-1)} \to \H$ have rank 1. 
\end{enumerate}
\end{prop}

Separable quantum states have no quantum entanglement.
Detecting separable states thus becomes equivalent to detecting quantum entanglement. Therefore, one way to determine if a state is separable is to compute the compact form of the SVD of every 1-flattening: if any 1-flattening has more than 1 significant singular value, the state is not separable. 

To generate training data, we first sample the space of separable states by pushing uniform samples on a product of spheres through the Segre map:
\[\mathcal{S}\H \times \cdots \times \mathcal{S}\H \to \mathcal{S}\Hn,\]
where $\mathcal{S}(\cdot )$ denotes the states of unit norm. The class of separable states is labeled as the class `0' and represents 50\% of the training dataset. The other half is the class of entangled states, labeled `1', and it is constructed by generating random tensors of rank greater than 2, by summing at least 2 random rank 1 tensors. Note that by Proposition~\ref{prop:separable} to certify the training set one can check the multilinear rank of the purported rank $\geq 2$ tensors to ensure they do not have rank 1 (and hence would be mis-classified). 


\subsubsection{Networks for separable states}
The set of separable states is the zeroset of the $2\times 2$ minors of 1-flattenings, and thus is defined by a set of homogeneous polynomials of degree 2. In the $2\times2$ case, there is only one defining equation, the $2\times2$ determinant. In higher dimensions, several equations of degree 2 define the Segre variety. To select the number of equations, we take (at least) the codimension of the Segre variety in the projective space.

For each quadratic equation, we can use the AH dimension count
 to deduce the number of needed neurons, and then to design the Hybrid Networks. In some cases, we chose to add several extra neurons in the first layer (with quadratic activation function) to accelerate the learning process. As explained before, the second layer will only contain neurons with LeakyReLU activation functions. The last layer will contain only one neuron, with a sigmoid activation function in this case (because of the binary classification problem).

We also implemented networks with only LeakyReLU activation functions neurons (excepted for the last layer). For all the considered cases, we used the same network architecture, which is composed of 4 first layers and a last output layer. We chose to implement a decreasing structure in the network, with 100 neurons in the first layer, 50 in the second, 25 in the third, and 16 in the fourth.
For each case the architecture of the network is provided in the second column of the results tables.

\subsubsection{Results}

In Table~\ref{Segre_Hybrid_results} we present the result for the $2\times2$, the $2\times2\times2$, the $2^{\times 4}$, and the $3\times3\times3$ cases with hybrid networks. In the $2\times2$, the $2\times2\times2$, and the $3\times3\times3$ cases, we used a training dataset of size 56200, a validation dataset of size 12800, and a testing dataset of size 32000. In the $2^{\times 4}$ case, we doubled the size of the training and the validation datasets only (same size for the testing one). We reached an average accuracy of 93\% on the testing datasets for separability classification.

\begin{table}[!h]
 \begin{center}
  \begin{tabular}{|c|c|c|c|c|c|}
\hline
Tensor size & Architecture & Training acc. & Validation acc. & Testing acc. & Loss  \\
\hline
$2\times2$ & (4,4,1) & 96.65\% & 96.60\% & 96.63\% & 0.092\\
$2\times2\times2$ & (21,8,1) & 94.57\% & 94.06\% & 94.44\% & 0.15\\
$2^{\times 4}$ & (1188,8,1) & 91.72\% & 91.60\% & 91.33\% & 0.26\\
$3\times3\times3$ & (332,12,1) & 94.68\% & 92.89\% & 92.94\% & 0.15\\
\hline
 \end{tabular}
 \end{center}
\caption{Hybrid network architectures and accuracies for each tensor size, for separability classification.}
\label{Segre_Hybrid_results}
\end{table}

We used LeakyReLU networks to study the same cases as we did with the hybrid networks, with the addition of the $2^{\times 5}$ one, and we regrouped the results in Table~\ref{Segre_ReLU_results}. In the $2\times2$ and $2\times2\times2$ cases, we used a training dataset of size 102400, a validation dataset of size 25600, and a testing dataset of size 32000. In the $2^{\times 4}$, $2^{\times 5}$ and $3\times3\times3$ cases, we used a training dataset of size 502600, a validation dataset of size 55600, and a testing dataset of size 32000. We reached an average accuracy of 98\% on the testing datasets for separability classification.
These results show that ANNs can be trained to distinguish between separable and entangled states for small to moderate sized tensors. 

\begin{table}[!h]
 \begin{center}
  \begin{tabular}{|c|c|c|c|c|c|}
\hline
Tensor size & Architecture & Training acc. & Validation acc. & Testing acc. & Loss  \\
\hline
$2\times2$ & (100,50,25,16,1) & 98.92\% & 98.78\% & 98.83\% & 0.043 \\
$2\times2\times2$ & (100,50,25,16,1) & 97.80\% & 97.42\% & 97.55\% & 0.074\\
$2^{\times 4}$ & (100,50,25,16,1) & 99.62\% & 99.50\% & 99.53\% & 0.016 \\
$2^{\times 5}$ & (100,50,25,16,1) & 98.83\% & 98.55\% & 98.55\% & 0.037 \\
$3\times3\times3$ & (100,50,25,16,1) & 98.55\% & 98.01\% & 97.92\% & 0.051  \\
\hline
 \end{tabular}
 \end{center}
\caption{LeakyReLU network architectures and accuracies for each tensor size, for separability classification.}
\label{Segre_ReLU_results}
\end{table}

\subsubsection{Complexity and efficiency} \label{sec:complexitySVD}
From an algebraic geometric point of view, determining if a tensor is on the Segre variety is equivalent to deciding if it is a rank one tensor, and this task can be performed using truncated SVD for each flattening. The complexity of computing the singular values of an $m\times n$ matrix is $\mathcal{O}(mn\min\{m,n\})$. This is roughly $\min\{m,n\}$ times the complexity of reading the matrix. 
Computing the singular values of the $i$-th flattening of a tensor of format $n_1\times \dots \times n_t$ has complexity
\[
\mathcal{O}\left(n_i (\frac{n_1\cdots n_t}{n_i}) \min \{n_i ,(\frac{n_1\cdots n_t}{n_i})\} \right)= \mathcal{O}\left((n_1\cdots n_t)\min \{n_i ,(\frac{n_1\cdots n_t}{n_i})\right)
.\]
For ``balanced'' cases $2n_i \leq \sum_i n_i$, thus this complexity is at most 
\[
 \mathcal{O}\left((n_1\cdots n_t)n_i\right),
\]
which is roughly $n_i$ times the complexity of reading the tensor. Now computing all of the singular values of the tensor (in a naive way, not informing one computation by the results of another) one would expect complexity $ \mathcal{O}\left((n_1\cdots n_t)(\sum_i n_i)\right)$. 
The cost of evaluating a trained network is equal to the cost of a feed-forward propagation, which needs to compute all the weighted sums (which is equivalent to compute matrix multiplications) and evaluate the activation function for each neuron. Hence, this complexity depends on the architecture of the network. If we denote by $l$ the number of layers (without counting the input layer) and by $m_1, m_2, \dots, m_l$ respectively the numbers of neurons in each layer, it has a complexity roughly equal to $\mathcal{O}((n_1\cdots n_t)m_1 + m_1m_2 + \cdots + m_{l-1}m_l)$.

So, as the dimensions of the tensor spaces increase, the cost of evaluating an already trained network should be more efficient than the computation of the SVD of flattenings. Moreover, flattenings cease to detect tensor ranks above the size of the flattening, whereas an ANN could be trained to detect higher ranks. Though the cost of training the network cannot be ignored, it is a one-time cost, and the trained network can be used for efficient computation for any number of tests. In addition, once the tensor rank is larger than the dimension of any of the factors, there are few effective methods for determining rank, so a trained ANN classifier may provide valuable insight.

\subsection{Detecting degenerate states}\label{dual_classifier}
For matrices, the opposite of  \emph{rank-1} is \emph{co-rank-1}. 
Recall that a matrix $A\in \RR^{m\times n}$ is \emph{degenerate} (rank deficient or co-rank-1) if and only if any of the following equivalent conditions hold:
\begin{enumerate}
\item  There are non-zero vectors $u \in \RR^m$ and $v \in \RR^n$ such that $u^\top Ax = 0$ for all $x \in \RR^n$   and $y^\top Av = 0$ for all $y \in \RR^m$. 
\item In the case $m=n$, the determinant vanishes: $\det(A) = 0$.
\item Up to Gaussian elimination $A$ has a zero row and a zero column. 
 \end{enumerate}

These conditions carry over in the tensor setting, following \cite{GKZ}. A state  $\ket{\phi} \in \Hn$ is \emph{degenerate} if one of the following equivalent conditions hold:
\begin{enumerate}
\item  There is a pure state  $x = \ket{x_1}\otimes \cdots \otimes \ket{x_n} \in \Hn$ such that  the contraction
\[
\left(\bra{x_1}\otimes \cdots \otimes \bra{x_{i-1}} \otimes 
\bra{h_i} \otimes \bra{x_{i+1}} \otimes \cdots \otimes \bra{x_n}\right) \ket {\phi} =0
\quad \text{for all }i \;\;\text{and for all} \; \ket{h_i} \in \H
\] 
\item  The hyperdeterminant vanishes:  $\textrm{Det}(A) = 0$, where $A$ is the hypermatrix satisfying $A_I = \braket{I}{\phi}$. 
\item There is an SLOCC equivalent state with coordinates $\phi_I = 0$  for all $I$ of Hamming distance $\leq 1$ away from $00\cdots 0$. 
 \end{enumerate}

The hyperdeterminant is a SLOCC invariant polynomial, and thus is used to characterize quantum entanglement. The zero locus of the equation defined by the hyperdeterminant also defines the dual of the Segre variety \cite{Miyake2002}, and the study of singularities of the associated hypersurfaces gives also a qualitative information about quantum entanglement \cite{Holweck2014a,Holweck2016}. 

The hyperdeterminant can also be regarded as a quantitative measure of entanglement and can be used to study entanglement in quantum algorithms for example \cite{Holweck2016a,Jaffali2019}. In a concrete sense the non-degenerate states are the most entangled ones and maximizing the absolute value of the hyperdeterminant can lead to maximally entangled states \cite{Gour2010,Chen2013,Holweck2019}.  While the hyperdeterminant is a polynomial test for degeneracy, it has very high degree and rarely feasible to compute apart from the smallest cases. We note that one of us showed that hyperdeterminants coming from the root system $E_8$ can be computed efficiently \cite{HolweckOedingE8}.

We wish to show that we can still learn degenerate states even in cases where an expression or evaluation method for the hyperdeterminant is not known explicitly. 

\subsubsection{A note on uniform sampling}
The SLOCC description of degenerate states gives a way to uniformly sample the set. We take a random state $\ket{\phi}$ with coordinates $\phi_I = 0$  for all $I$ of Hamming distance $\leq 1$ away from $(00\cdots 0)$, we renormalize so that $||\phi|| = 1$, and then we apply a random element of the SLOCC group. The degenerate states form the half of the training data and are labeled as the class `0'. On another hand, we generate random tensors in $\Hn$ to build the second half of the training dataset, which correspond (with high probability\footnote{The justification of the term ``high probability'' is the following. For real numbers, the set of degenerate tensors has codimension at least 1 in the ambient space $\RR^{d^{n}}$ and thus it has measure zero. Thus, the probability is zero that a tensor chosen randomly from $\RR^{d^{n}}$ lands in a measure zero set. When floating point precision is used one expects that these continuous concepts are still well-approximated in the discrete case, even despite the fact that a non-empty subset (the exceptions) never has measure zero in this case.}) to non-degenerate states labeled as the class `1'. When an analytical method for evaluating the hyperdeterminant is available, we used it to remove all the noise from the training data (remove all random states that are not non-degenerate).

\subsubsection{Networks for degenerate states}
The $2^{\times n}$ hyperdeterminant is a homogeneous polynomial whose degree is $4, 24,$ or $128 $ respectively for $n=3,4,$ or $5$, and its degree is 36 in the $3\times 3\times 3$ case (see \cite{GKZ} for concrete degree formulas). 

In the $2\times 2\times 2$ case, the first layer of the hybrid network must contain at least $\ceil{\frac{1}{8}\binom{4+8-1}{4}}=42$ neurons with $x \mapsto x^4$ activation functions. We chose to implement 60 neurons in the first layer for the 3-qubit case.
In the two other cases, the number of needed neurons is way too high to hope for an implementation using Keras. One can still use fewer neurons to approximate the hyperdeterminant using hybrid networks, and even if the accuracy will not be very close to 100\%, on can still reach 60\% or more and use the idea presented in Section~\ref{prediction_application} to have a better overall accuracy. 

The second possibility is to use only LeakyReLU activation functions (except for the output layer), as it was done in Subsection~\ref{segre_classifier}. We also used the same decreasing structure with four layers, except in the 4-qubit case, when we chose to double the number of neurons in the three first layers. 

 The $2\times2$ non-degenerate states are in fact non-separable states. This case is thus already solved by the $2\times2$ separability classifiers presented in Section~\ref{segre_classifier}.

\subsubsection{Results}

In Table~\ref{Dual_Hybrid_results} we present the result for the $2\times2\times2$ case with the hybrid network. We used a training dataset of size 202600, a validation dataset of size 25600, and a testing dataset of size 32000. We reached 92\% of accuracy for the testing dataset.

\begin{table}[h!]
 \begin{center}
  \begin{tabular}{|c|c|c|c|c|c|}
\hline
Tensor size & Architecture & Training acc. & Validation acc. & Testing acc. & Loss  \\
\hline
$2\times2\times2$ & (60,10,1) & 92.49\% & 92.18\% & 92.09\% & 0.1837\\
\hline
 \end{tabular}
 \end{center}
\caption{Hybrid network architectures and accuracies for each tensor size, for degenerate and non-degenerate states classification.}
\label{Dual_Hybrid_results}
\end{table}

In Table~\ref{Dual_ReLU_results} we present the results for the $2\times2\times2$, $2^{\times 4}$, $2^{\times 5}$, and $3\times3\times3$ cases with networks using LeakyReLU activation functions. In the $2\times2\times2$ case, we used a training dataset of size 502600, a validation dataset of size 55600, and a testing dataset of size 32000. In the $2^{\times 4}$ and $2^{\times 5}$ cases, we used a training dataset of size 252400, a validation dataset of size 55600, and a testing dataset of size 52000. In the $3\times3\times3$ case, we used a training dataset of size 352400, a validation dataset of size 55600, and a testing dataset of size 52000. We reached an average accuracy of 96\% on the testing datasets for classification of degenerate states.

These results are quite interesting, especially in the 5-qubit case. In fact, as mentioned before, no explicit expression or efficient evaluation method for the $2^{\times 5}$ hyperdeterminant is available, and it is expected to be intractable. The network has 98\% accuracy for the testing dataset, which means that we can determine with very high accuracy if a tensor is degenerate or not, which is equivalent to determining if the hyperdeterminant vanishes or not. In the context of qualitative characterization of entanglement, one only needs to know if the state annihilates or not the hyperdeterminant. In this sense, we provide an original result for the evaluation of the nullity of the hyperdeterminant for $2^{\times 5}$ real tensors, that can be used to investigate quantum entanglement for 5-qubit systems (see Section~\ref{sec:5qubit}). 

\begin{table}[!h]
 \begin{center}
  \begin{tabular}{|c|c|c|c|c|c|}
\hline
Tensor size & Architecture & Training acc. & Validation acc. & Testing acc. & Loss \\
\hline
$2\times2\times2$ & (100,50,25,16,1) & 93.44\% & 92.53\% & 92.74\% & 0.1629 \\
$2^{\times 4}$ & (200,100,50,16,1) & 99.50\% & 95.95\% & 95.94\% & 0.01791\\
$2^{\times 5}$ & (100,50,25,16,1) & 99.95\% & 98.74\% & 98.83\% & 0.001533\\
$3\times3\times3$ & (100,50,25,16,1) & 98.18\% & 96.78\% & 96.83\% & 0.04770\\
\hline
 \end{tabular}
 \end{center}
\caption{LeakyReLU network architectures and accuracies for each tensor size, for degenerate and non-degenerate states classification.}
\label{Dual_ReLU_results}
\end{table}

\subsection{Border-rank classification} \label{rank_classifier}
A state $\ket{\phi} \in \Hn$ is said to have \emph{rank} $\leq R$ if there is an expression $\ket{\phi} = \sum_{r=1}^R \lambda_i\ket{\phi_i}$ with $\lambda_i \in \H$ and $\ket{\phi_i} \in \Hn$ \cite{CarliniGrieveOeding}. It is well known that when $n>2$ rank is not a closed condition, and it is not semi-continuous. The first example is 3-qubits, where the $W$-state $\ket{100} + \ket{010} + \ket{001}$, which has rank 3, and is in the closure of the generic orbit, that of $\ket{000}+\ket{111}$, which has rank 2 \cite{Holweck2012}. 

In another context, Bini \cite{Bini_BorderRank} defined the notion of border rank to regain semi-continuity. One says that a state $\ket{\phi} \in \Hn$ has \emph{border rank} $\leq R$ if there exists a family of rank $R$ states $\{\ket{\phi^\epsilon}\mid \epsilon >0\}$ and $\lim_{\epsilon \to 0} \ket{\phi^\epsilon} = \ket{\phi}$. Equivalently, the set of border rank $\leq R$ states, denoted $\sigma_R$ is the Zariski closure of the states of rank $\leq R$. 

By construction, when $\H = \CC^d$  we have a chain
\[
\sigma_1 \subsetneq \sigma_2 \subsetneq \cdots \subsetneq \sigma_g = \Hn
\]
which ends at $\Hn$ because the set of separable states is linearly non-degenerate in $\Hn$. The minimal integer $g$ for which $\sigma_g = \Hn$ is called the \emph{generic rank}. A simple dimension count gives the \emph{expected generic rank}:
\[
e=\left\lceil
\frac{d^n}{n(d-1)+1}
\right\rceil.
\] 

For complex tensors of format $d^{\times n}$ the only known case when the generic rank differs from the expected rank is $3\times 3\times 3$ tensors or 3-qutrits ($d=3, n = 3$ in the formulation above). Here the generic rank is 5 and not 4 as expected. For $2\times 2\times 2\times 2$ tensors or 4 qubits (the case $d=2, n = 4$ above) the generic rank equals the expected generic rank of 4, even though it is known that the set of rank 3 tensors is defective. It is conjectured that these are the only such cases  (see \cite{BaurDraismadeGraaf, AOP_Segre, ChiOttVan}).

Let $\sigma_s^\circ$ denote the states of rank exactly $s$. 
When $\H = \RR^d$ there can be more than one semi-algebraic set $\sigma_s^\circ$ that is full dimensional. These ranks are called \emph{typical ranks}.  Note that for complex numbers there is only one typical rank, which is the \emph{generic rank}.
Given a probability measure $\mu$ on $\Hn$ the measure of $\sigma_s^\circ$ represents the probability of a random state having rank $s$. 
The following example from \cite{deSilva_Lim_ill_posed} illustrates illustrates this situation.
\begin{example}\label{example_typical_rank}Set $n=3$, $d=2$. Both 2 and 3 are typical ranks, and moreover a positive volume set of tensors with rank 3 have no optimal low-rank approximation. The $2\times 2\times 2$ hyperdeterminant $D$ separates $\RR^{2\times 2\times 2}$ into regions of constant typical rank. 
We formed a sample of tensors of real rank at most 3 as a normalized sum of tensor products of vectors with entries uniformly distributed in $[-.5,.5]$.
For this distribution we found:
\begin{enumerate}
\item with approximate frequency 86.6\% $D(\phi) >0$, in which case the $\RR$-rank is 2,
\item with approximate frequency 13.4\% $D(\phi) <0$, in which case $\RR$-rank is 3,
\item with frequency 0\% $D(\phi) =0$, in which case the $\RR$-rank can be 0,1,2,3.
\qedhere\end{enumerate}
\end{example}

\subsubsection{A note on uniform sampling}
We construct points on algebraic varieties by utilizing rational parameterizations of the form $\phi\colon U \to V $, with $\phi$ a rational function, and $U$, $V$ subsets of a normed linear space.  In general, if $S$ is a uniform sample of $U$, $\phi(S)$ will not be a uniform sample in $\phi(U)$.
Instead of trying to uniformly sample the image of these varieties \cite{Hernandez-Mederos2003,Pagani2018,Dufresne2018}, we provide training data that is constructed in the way we think that one might construct the test data, or how one might imagine the data is produced from a state constructed in a lab. Whether this is a reasonable assumption is open for debate.  On one hand, it is always possible (by forcing incoherence, for instance) for an adversarial entity to construct a data set on the same variety that fools our classifier. 
On the other hand, we can say that if we train our neural network on data produced by a certain process, we can be confident that if the validation data is constructed by the same process, then the accuracy numbers we report reflect a reasonable measure of confidence in the classifier.


\subsubsection{Results}


In the $2\times2\times2$ case, we trained our networks to recognize the exact rank of tensors. In fact, for building the training and validation datasets, we provided tensors from Segre variety for rank one, tensors that are SLOCC equivalents to the biseparable and GHZ states for rank 2, and states that are SLOCC equivalents to the W state for rank 3, following the classification of 3-qubits (see Table~\ref{3qubits_classification}). We also took into account the typical rank (see Example~\ref{example_typical_rank}) and used the sign of the hyperdeterminant to separate rank 2 and 3. 

In the $2^{\times 4}$ and $2^{\times 5}$ cases, we generated tensors for each class `$k$' by summing $k+1$ rank one tensors (and renormalized). The class `0' corresponds to rank one tensors, and the class `$r$' correspond to tensors with border rank $r+1$. The training dataset is equally divided such that it contains the same number of samples for each class. In the $2^{\times 4}$, we generated tensors up to border rank 4, and up to border rank 5 in the $2^{\times 5}$ case. By generating vectors for each class this way, we introduce noise in the data, especially in the real case. However, even in the presence of noise, the network is able to learn and predict border rank of states with good accuracy (84\% in the 4-qubit case, and 80\% in the 5-qubit case, on the testing dataset, which is also generated by the same process). 

For example, we repeatedly evaluated the 5-qubit network on the input states SLOCC equivalent to $\ket{W_5} = \ket{00001}+\ket{00010}+\ket{00100}+\ket{01000}+\ket{10000}$  (see Section~\ref{prediction_application}), and most of the time we (correctly) obtain the class `1' (see Figure~\ref{fig:W5_rank}), indicating border rank 2. 

Table~\ref{Rank_Hybrid_results} contains our results for the $2\times2\times2$ rank classification with the hybrid network. We used a training dataset of size 102400, a validation dataset of size 25600, and a testing dataset of size 32000. We reached 87\%  accuracy for the testing dataset.

\begin{table}[!h]
 \begin{center} 
  \begin{tabular}{|c|c|c|c|c|c|}
\hline
Tensor size & Architecture & Training acc. & Validation acc. & Testing acc. & Loss \\
\hline
$2\times2\times2$ & (169,25,3) & 88.19\% & 88.03\% & 87.95\% & 0.3028\\
\hline
 \end{tabular}
 \end{center}
\caption{Hybrid network architectures and accuracies for 3-qubits, for rank classification.}
\label{Rank_Hybrid_results}
\end{table}

Table~\ref{Rank_ReLU_results} contains results for the $2\times2\times2$, $2^{\times 4}$, and $2^{\times 5}$  cases using the LeakyReLU network. In the $2\times2\times2$ rank classification problem, we used a training dataset of size 502500, a validation dataset of size 55600, and a testing dataset of size 52000. In the $2^{\times 4}$ and $2^{\times 5}$ border rank classification problems, we respectively used a training dataset of size 502400 and 802400, a validation dataset of size 55600 and a testing dataset of size 52000.

\begin{table}[!h]
 \begin{center}
  \begin{tabular}{|c|c|c|c|c|c|}
\hline
Tensor size & Architecture & Training acc. & Validation acc. & Testing acc. & Loss  \\
\hline
$2\times2\times2$ & (200,100,50,25,3) & 94.19\% & 94.07\% & 93.79\% & 0.1674\\
$2^{\times 4}$ & (200,100,50,25,3) & 85.49\% & 84.45\% & 84.47\% & 0.3144\\
$2^{\times 5}$ & (200,100,50,25,3) & 81.39\% & 79.88\% & 79.77\% & 0.4230\\
\hline
 \end{tabular}
 \end{center}
\caption{LeakyReLU network architectures and accuracies for each tensor size, for rank and border-rank classification.}
\label{Rank_ReLU_results}
\end{table}

\section{How to use the classifier after training}\label{sectionApplication}

\subsection{Predictions for a single quantum state}\label{prediction_application}
We study quantum entanglement from a qualitative point of view. Quantum states are regrouped into classes, which correspond to SLOCC orbits. When the number of orbits is infinite, we talk about families depending on parameters \cite{Holweck2017}. Any state belonging to a specific class will be equivalent to another state in the same class by the action of an element of this group. Moreover, other properties, such as separability, degeneracy, and tensor rank, are invariants under the SLOCC action \cite{Holweck2012}. Hence, our classifiers should give the same answer for SLOCC equivalent states. 

However, the accuracy of our classifiers is not 100\%, and thus the probability of misclassification is non-zero, which gives little confidence for a single test. In order to significantly reduce misclassification for a specific quantum state (or set of states) we generate SLOCC equivalent states and look at the most frequent answer of the neural network classifier. A histogram of these results gives a graphical representation of the neural network output distribution. Examples are investigated in the next sections. We report the results after applying the classifier to enough points of the orbit so that the consensus was reached (the answer was the same for larger data samples). 

\begin{remark}
In the case of an experimental implementation of our classifiers, this process of generating datasets for predictions should be put into perspective with the impossibility of generating several copies (due to the non-cloning theorem \cite{Nielsen2011}) of the same state, before applying a random local SLOCC transformation \cite{Xu2008,Lanyon2009}.
\end{remark}

\subsection{Three-qubit entanglement classification}\label{sec:3qubit}

The 3-qubit case is simplest case that illustrates the existence of non-equivalent entangled states, such as the $\ket{GHZ}$ and $\ket{W}$ states, as it was noted in \cite{Duer2000}. The entanglement classification of 3-qubit systems (see Table~\ref{3qubits_classification}) can be described by join, secant and tangential varieties \cite{Holweck2012}, and by using dual varieties as well \cite{Holweck2012,Miyake2002,Miyake2003,Miyake2004}. Rank can also be used as an algebraic measure of entanglement to distinguish between several 3-qubits entanglement classes \cite{Brylinski2000}. 
\begin{table}[htp] 
\centering
 \begin{center}
  \begin{tabular}{|c|c|c|c|}
\hline
Normal forms & Class & Rank & Cayley's $\Delta_{222}$  \\
\hline
$\ket{000}+\ket{111}$ & GHZ & 2 & $\neq 0$\\
$\ket{001}+\ket{010}+\ket{100}$ & W  & 3 &  0 \\
$\ket{001}+\ket{111}$ & Bi-Separable C-AB & 2 & 0\\
$\ket{010}+\ket{111}$ & Bi-Separable B-CA & 2 & 0 \\
$\ket{100}+\ket{111}$ &  Bi-Separable A-BC & 2 & 0 \\
$\ket{000}$ & Separable & 1 & 0 \\
\hline
 \end{tabular}
 \end{center}
\caption{Entanglement classification of 3-qubit systems under SLOCC.}
\label{3qubits_classification}
\end{table}

Our binary classifiers built in Section~\ref{sectionExperiments} can also be used to distinguish between 3-qubit entanglement classes. In fact, separable states can be detected using the binary classifier for tensors on the Segre variety, presented in Section~\ref{segre_classifier}. All states that are SLOCC equivalents to the $\ket{GHZ}$ state correspond to non-degenerate states, and thus can be recognized using the binary classifier for tensors on the dual variety of the Segre variety, presented in Section~\ref{dual_classifier}. Finally, to distinguish between Bi-Separable states and states that are SLOCC equivalents to $\ket{W}$, we exploit the rank of tensors, by using the rank classifier, presented in Section~\ref{rank_classifier}, as a predictor for distinguishing between rank 2 and rank 3 tensors. See Figures~\ref{fig:3qubits_sep},\ref{fig:3qubits_bisep},\ref{fig:3qubits_w},\ref{fig:3qubits_ghz}.

\subsection{Entanglement of five-qubit systems}\label{sec:5qubit}

In the five-qubit case, since no complete or parametrized classification is known, distinguishing entanglement classes is a difficult task. Some prior work focused on using filters \cite{Osterloh2006} or using polynomial invariants \cite{Luque2006}. In these papers, they considered four 5-qubit systems $\ket{\Phi_1}$, $\ket{\Phi_2}$, $\ket{\Phi_3}$, and $\ket{\Phi_4}$ (see Equations~\ref{eq:phi1},~\ref{eq:phi2},~\ref{eq:phi3},~\ref{eq:phi4}).
\begin{align}
\label{eq:phi1}
\ket{\Phi_1} &= {\textstyle\frac{1}{\sqrt{2}}} \left( \ket{00000} + \ket{11111} \right)
\\
\label{eq:phi2}
\ket{\Phi_2} &= {\textstyle\frac{1}{2}} \left( \ket{11111} + \ket{11100} + \ket{00010} + \ket{00001} \right)
\\
\label{eq:phi3}
\ket{\Phi_3} &= {\textstyle\frac{1}{\sqrt{6}}} \left( \sqrt{2}\ket{11111} + \ket{11000} + \ket{00100} + \ket{00010} + \ket{00001} \right)
\\
\label{eq:phi4}
\ket{\Phi_4} &= {\textstyle\frac{1}{2\sqrt{2}}} \left( \sqrt{3}\ket{11111} + \ket{10000} + \ket{01000} + \ket{00100} + \ket{00010} + \ket{00001} \right)
\end{align}
These states do not belong to the same entanglement classes \cite{Osterloh2006,Luque2006}. We ask if these states are degenerate or not. Our classifier for 5-qubit degenerate states predicts that each of these states is degenerate, hence the $2^{\times 5}$ hyperdeterminant should vanish on each (see Figure~\ref{fig:5qubits_phis}).

On the other hand, the states $\ket{\delta_1}$ and $\ket{\delta_2}$ (see Equations~\ref{eq:delta1} and \ref{eq:delta2}) are known to be non-degenerate \cite{Holweck2014a,Holweck2016}. Our classifier also predicts they are non-degenerate, see Figure~\ref{fig:5qubits_delta}. 

\begin{equation}\label{eq:delta1}
\begin{split}
\ket{\delta_1} = \; & {\textstyle\frac{1}{\sqrt{14}}} \big( \ket{00000} + \sqrt{3}\ket{00011} + \ket{00100} + \ket{01000} + \ket{01001} + \\ &\sqrt{2}\ket{01111} + \ket{10001} + \ket{10110} + \ket{11000} + \ket{11011} + \ket{11101} \big)
\end{split}\end{equation}
\begin{equation}\label{eq:delta2}
\begin{split}
\ket{\delta_2} =\;& {\textstyle \frac{1}{\sqrt{11}}} \big( \ket{00000} + \ket{00100} + \ket{00111} + \ket{01010} - \ket{01101} + \\ &\ket{10001} + \ket{10011} + \ket{10111} - \ket{11000} + \ket{11110} \big)
\end{split}\end{equation}

\section{Acknowledgements}
This work was partially supported by the R\'egion Bourgogne Franche-Comt\'e, project PHYFA (contract 20174-06235), the French “Investissements d’Avenir” programme, project ISITE-BFC (contract ANR-15-IDEX-03) and by ``BQR Mobilit\'e Doctorante'' at U.T.B.M. The first author thanks the Department of Mathematics and Statistics at Auburn University for its hospitality,and Dr. Mark Carpenter (Auburn University) and Dr. Frabrice Lauri (U.T.B.M.) for lectures and discussion about Machine Learning. The authors want to thank Dr. Fr\'ed\'eric Holweck for his advice, comments and discussions as well as the reviewers for their relevant comments and remarks, which improved the exposition of the article.
\newcommand{\arxiv}[1]{{\tt \href{http://arxiv.org/abs/#1}{{arXiv:#1}}}}

\bibliography{references}

\section*{Histograms}
\begin{figure}[!h]
\includegraphics[width=0.52\textwidth]{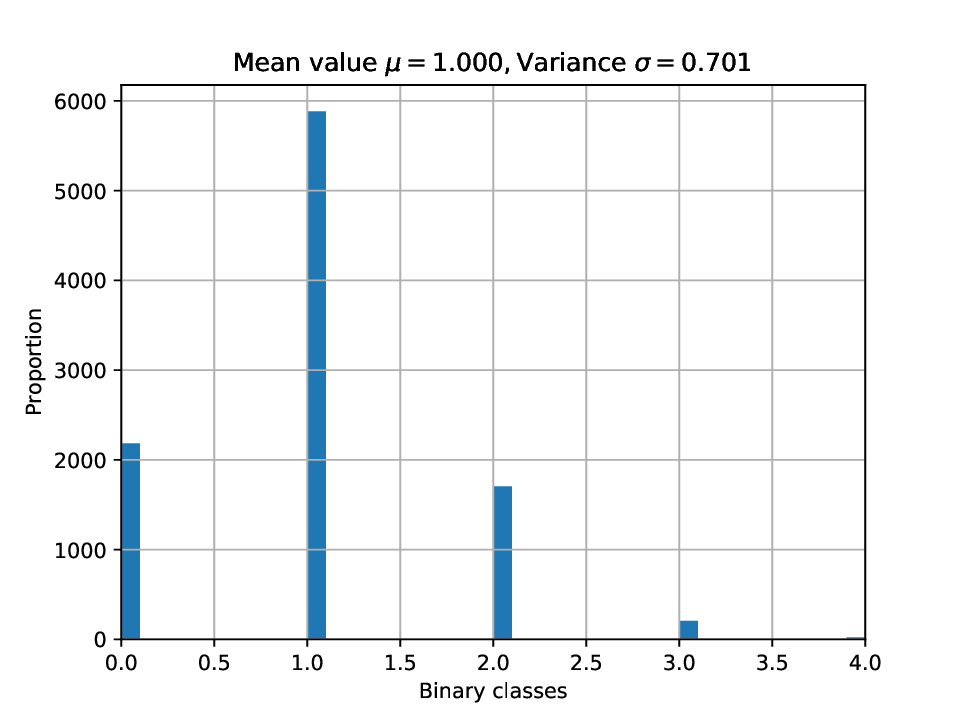}
\caption{Histogram of the border rank classifier predictions for 10000 points SLOCC equivalent to the state $\ket{W_5}$. The plot predicts that the state is of border rank 2 (class `1').} \label{fig:W5_rank}
\end{figure}

\begin{figure}[!h]
\includegraphics[width=0.327\textwidth]{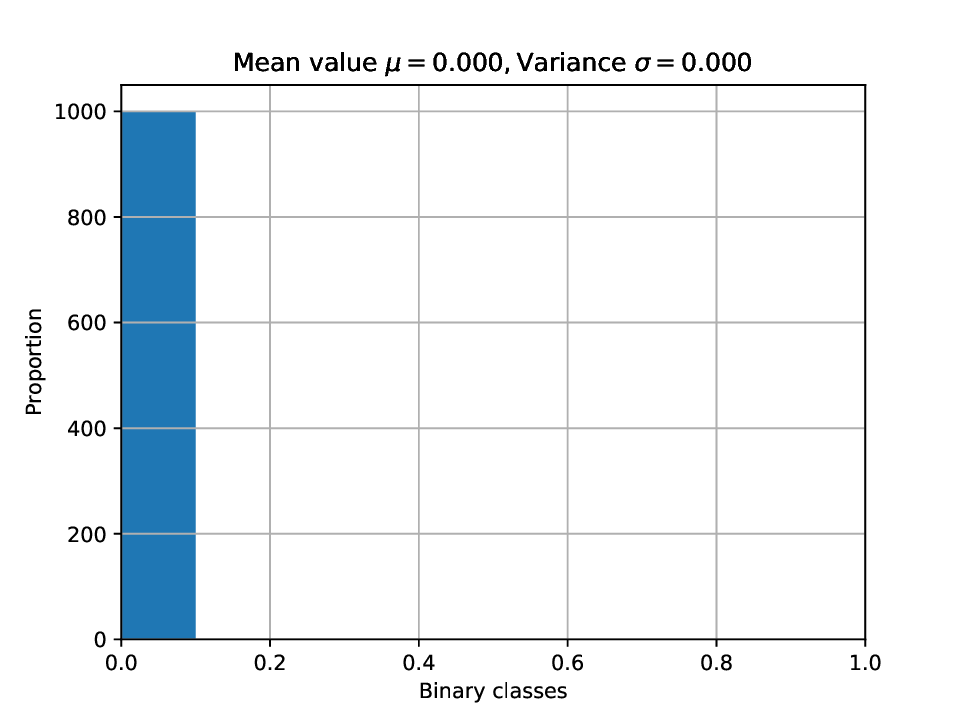}
\includegraphics[width=0.327\textwidth]{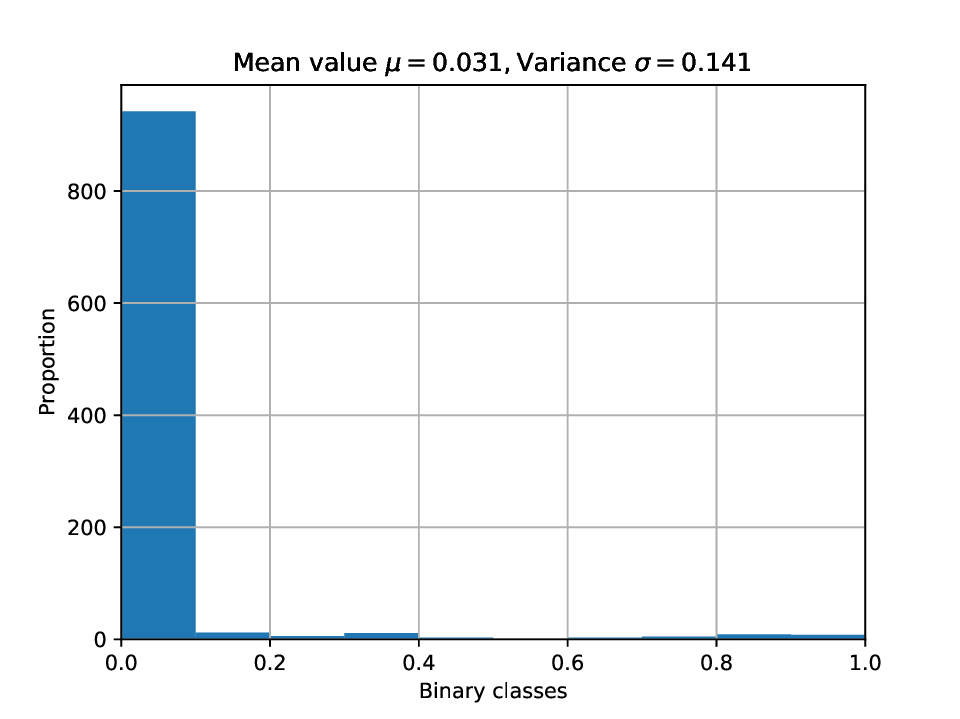}
\includegraphics[width=0.327\textwidth]{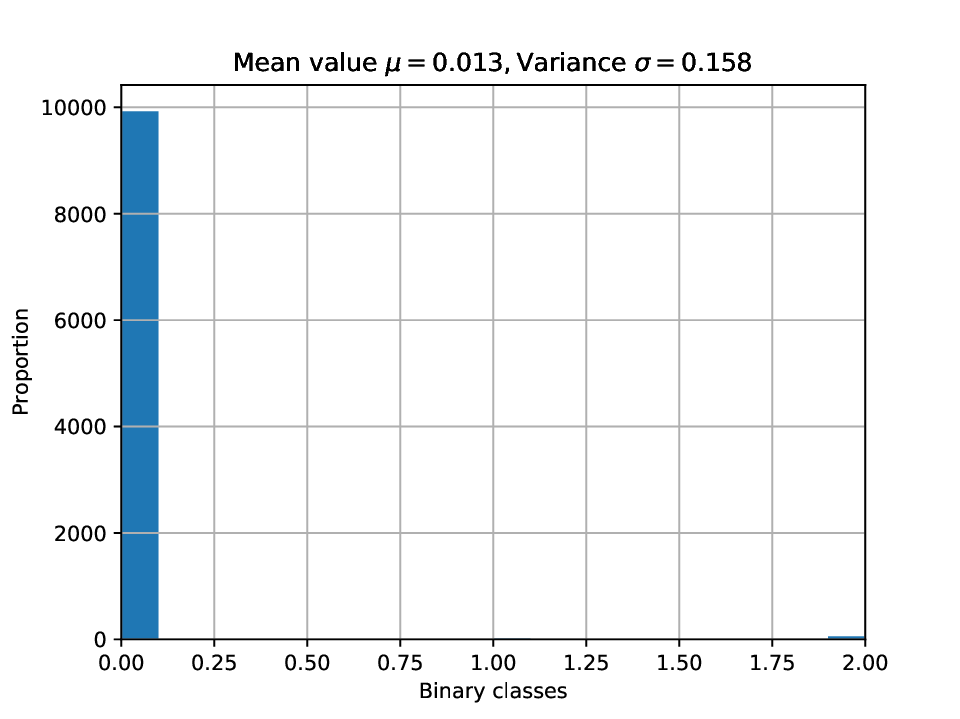}

\caption{Histogram predictions for 1000 points  SLOCC equivalent to $\ket{000}$ using our trained classifiers for (in order, from left to right) separable states, degenerate states and tensor rank. Being class `0' in each plot respectively predicts that the state is separable, degenerate, and of rank one.} \label{fig:3qubits_sep}
\end{figure}

\begin{figure}[!h]
\includegraphics[width=0.327\textwidth]{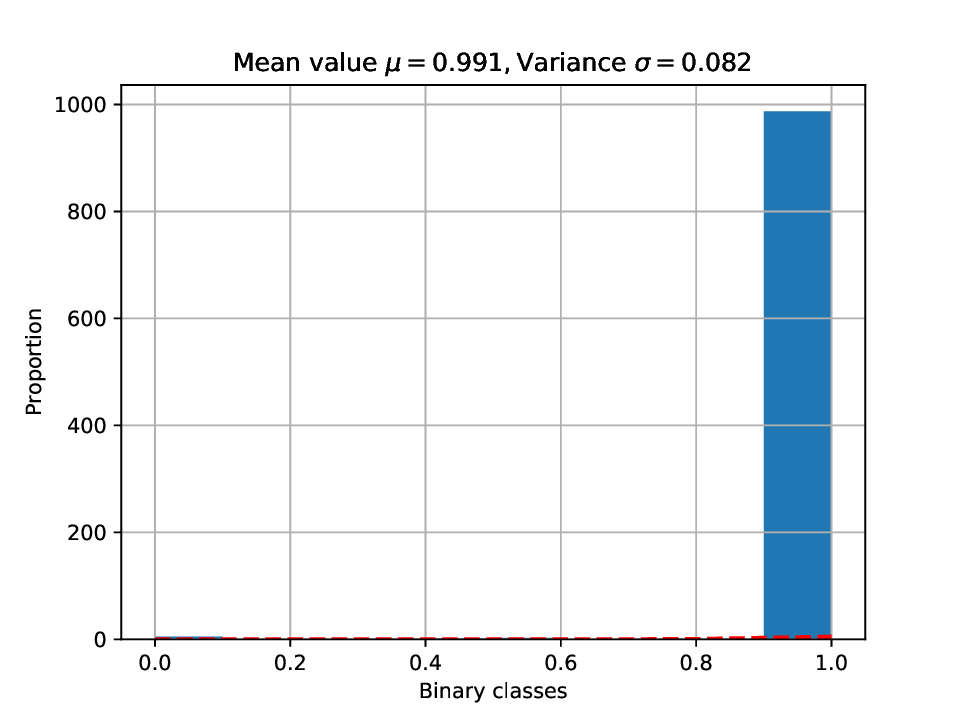}
\includegraphics[width=0.327\textwidth]{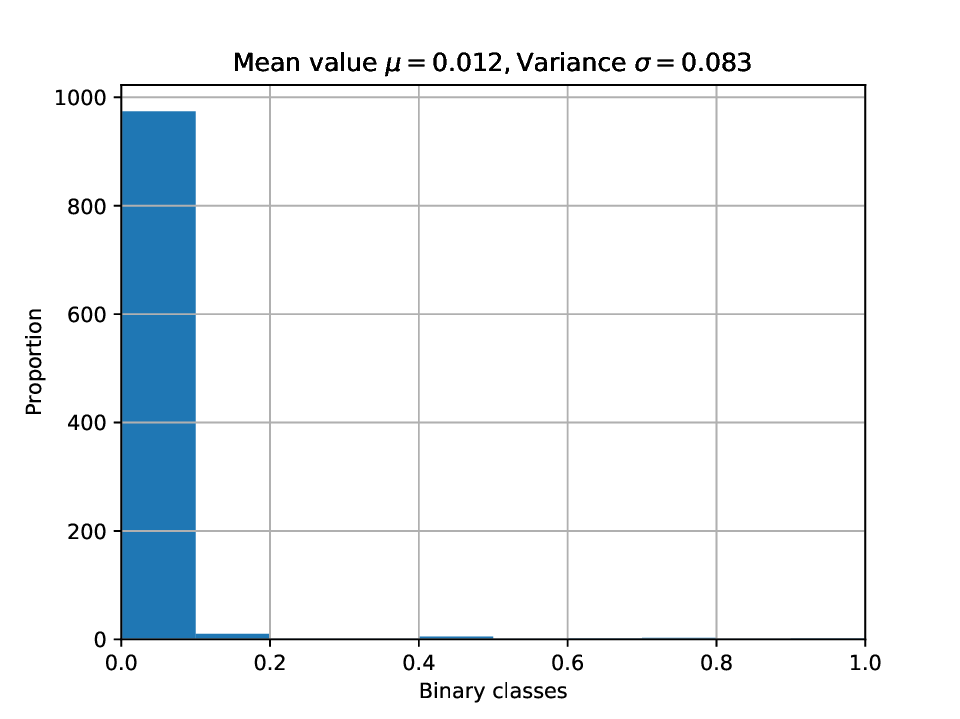}
\includegraphics[width=0.327\textwidth]{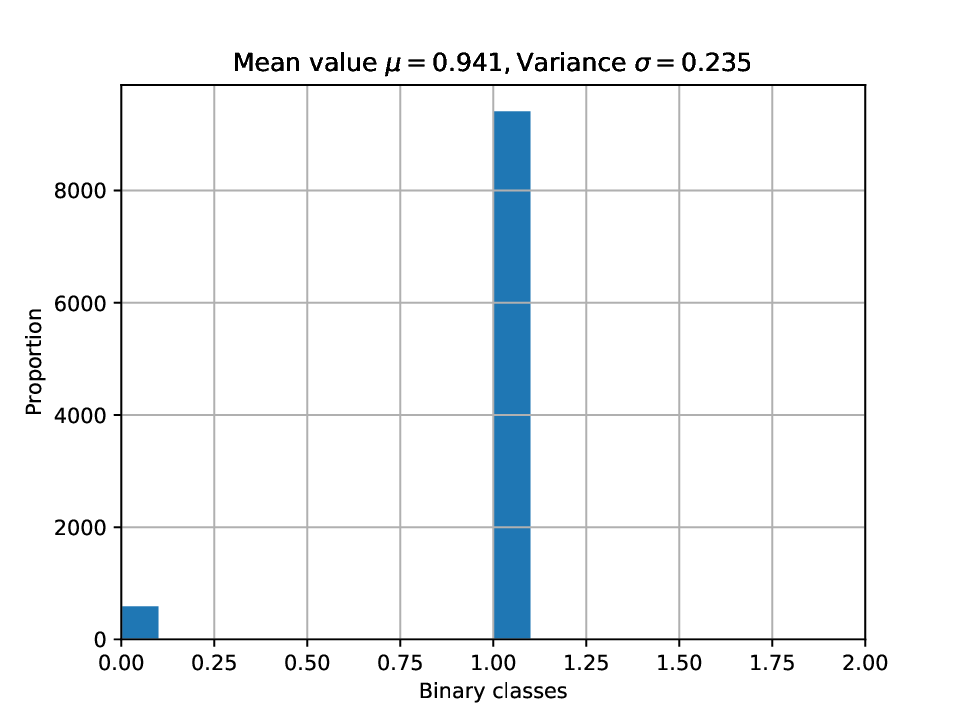}

\caption{Histogram predictions for 1000 points that are SLOCC equivalent to $\frac{1}{\sqrt 2}(\ket{000}+\ket{011})$ using our trained classifier for (in order, from left to right) separable states, degenerate states and tensor rank classifiers. The plots predict the state is entangled (class `1' on the left), degenerate (class `0' in the middle), and of rank two (class `1' on the right).} \label{fig:3qubits_bisep}
\end{figure}

\begin{figure}[!h] 
\centering
\includegraphics[width=0.327\textwidth]{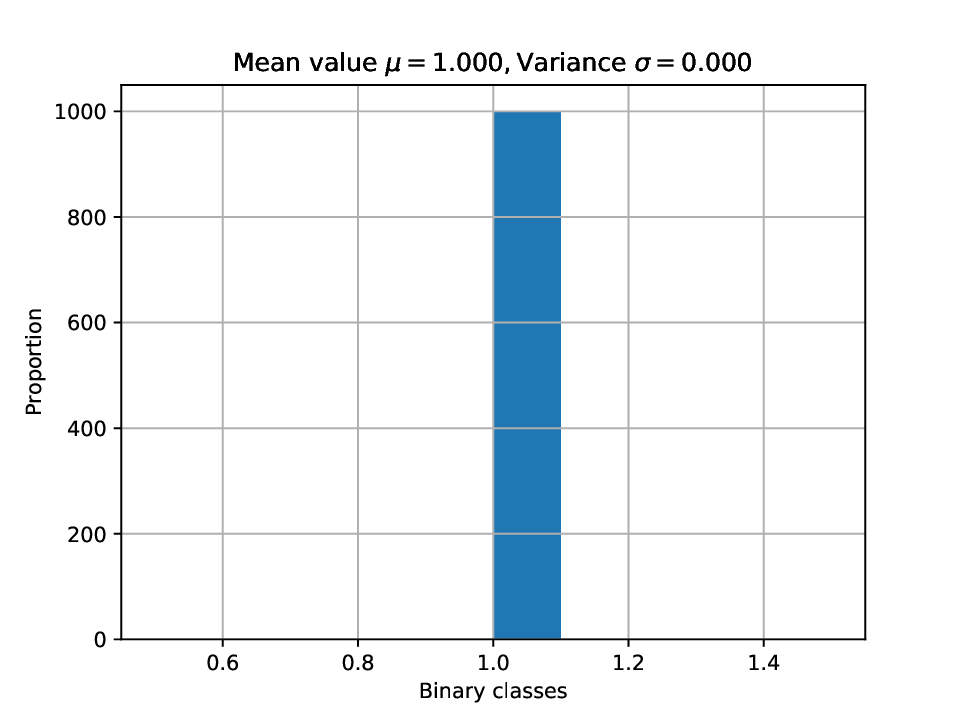}
\includegraphics[width=0.327\textwidth]{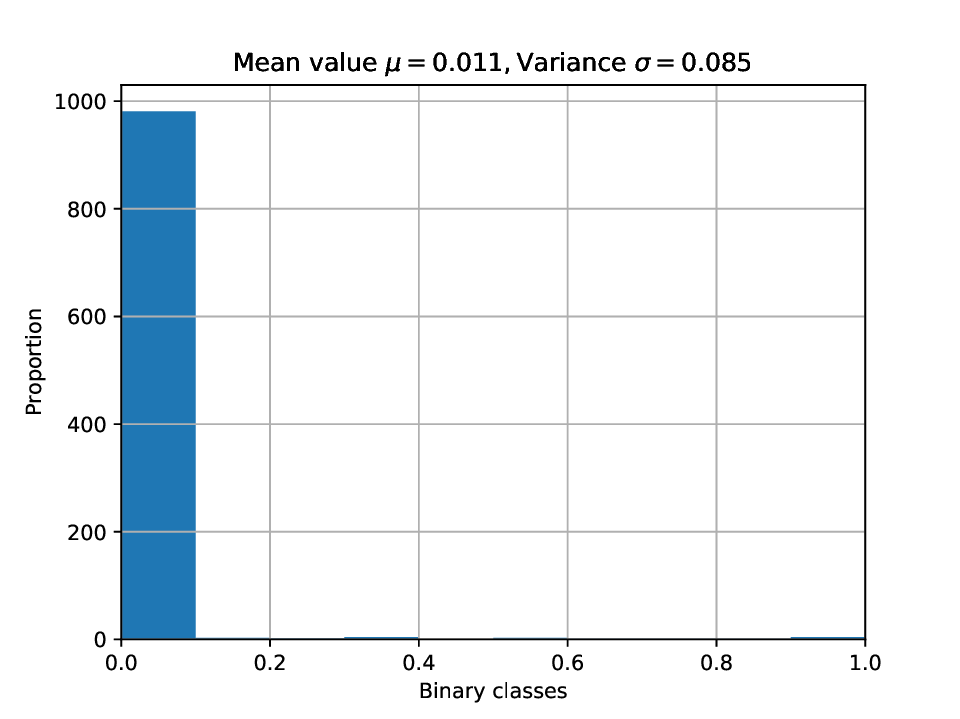}
\includegraphics[width=0.327\textwidth]{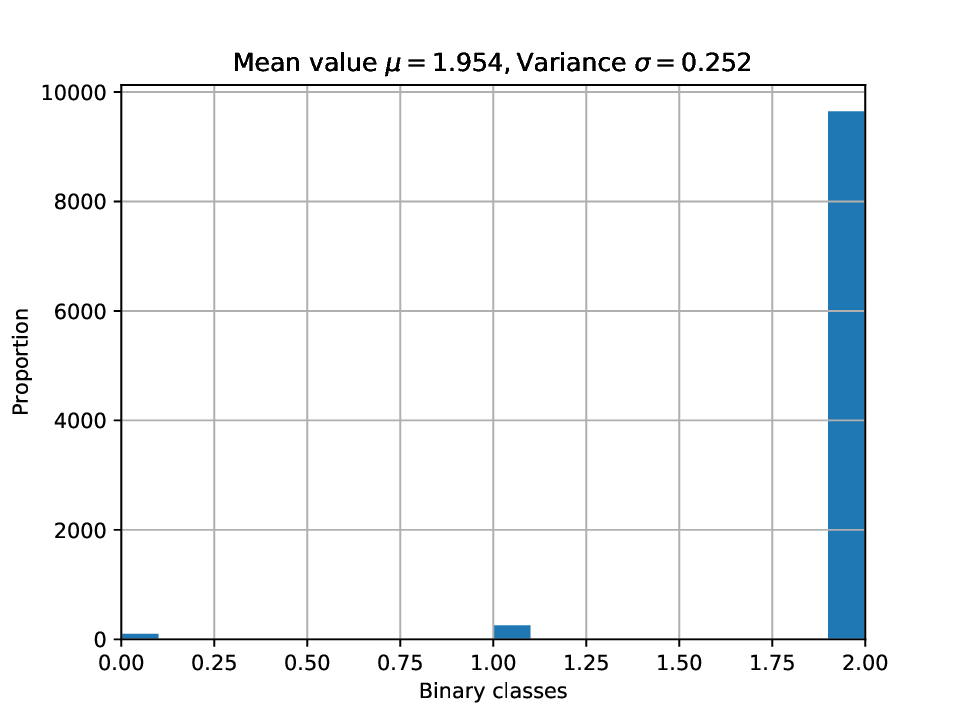}
\caption{
Histogram predictions for points that are SLOCC equivalent to the W-state $\frac{1}{\sqrt 3}(\ket{001}+\ket{010}+\ket{100})$ using our trained classifiers for (in order, from left to right) separable states, degenerate states and tensor rank. The left plot and middle plots use 1000 points and respectively predict that the state is entangled (class `1') and degenerate (class `0'). The right plot with 10000 equivalent points predicts that the state is of rank three (class `2').} \label{fig:3qubits_w}
\end{figure}

\begin{figure}[!h]
\includegraphics[width=0.327\textwidth]{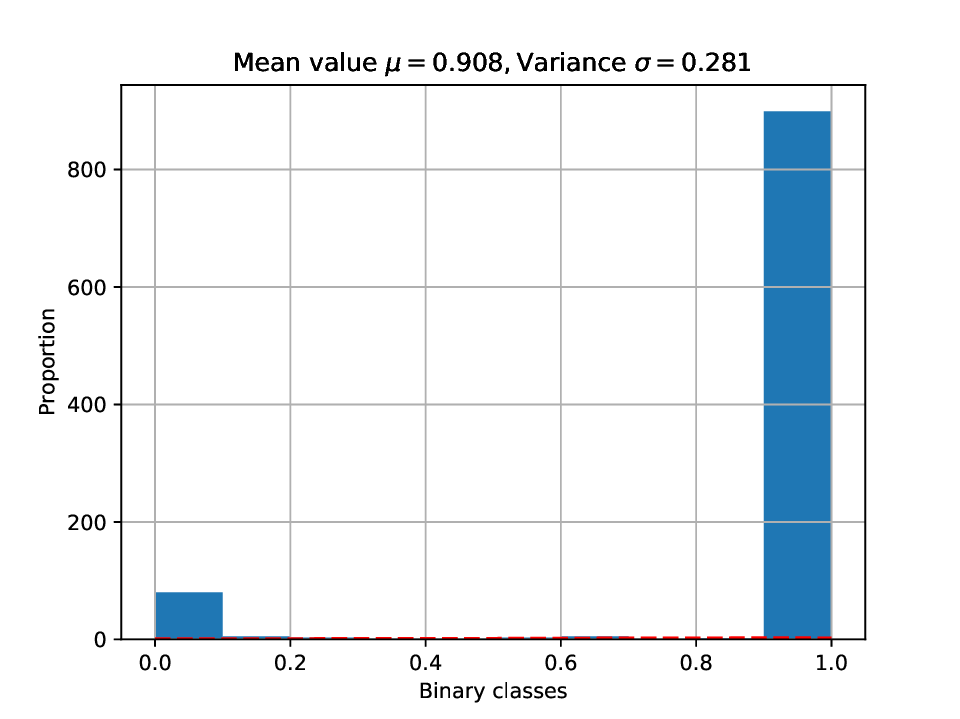}
\includegraphics[width=0.327\textwidth]{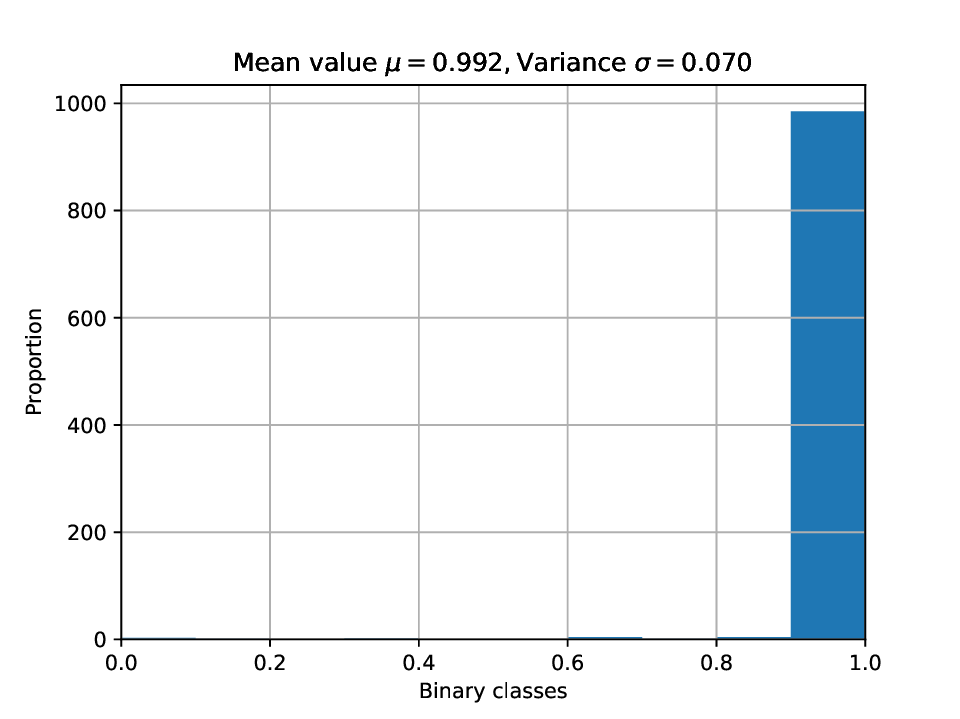}
\includegraphics[width=0.327\textwidth]{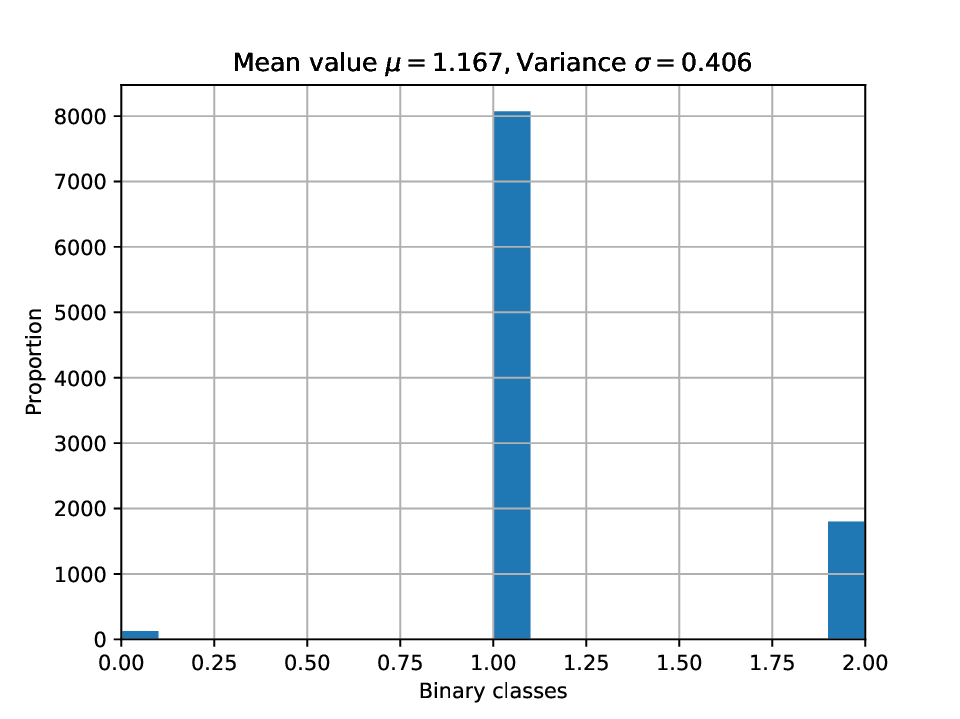}
\caption{
Histogram predictions for points that are SLOCC equivalent to the GHZ-state $\frac{1}{\sqrt 2}(\ket{000}+\ket{111})$ using our trained classifiers for (in order, from left to right) separable states, degenerate states and tensor rank. The left plot and middle plots use 1000 points and respectively predict that the state is entangled (class `1') and non-degenerate (class `1'). The right plot with 10000 equivalent points predicts that the state is of rank two (class `1').
} \label{fig:3qubits_ghz}
\end{figure}

\begin{figure}[!h]
\includegraphics[width=0.244\textwidth]{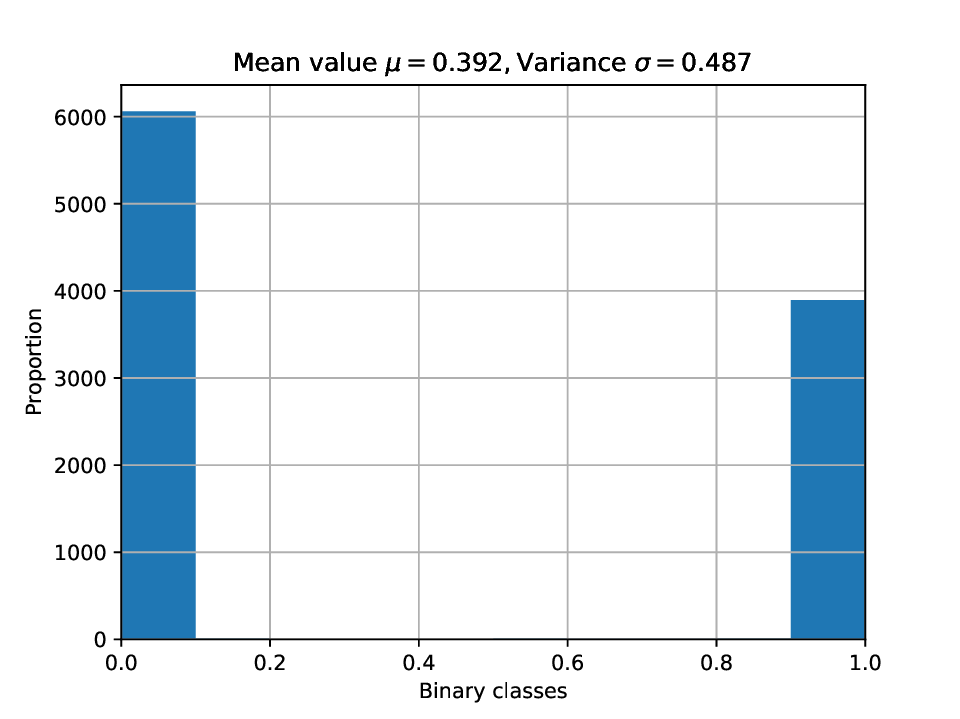}
\includegraphics[width=0.244\textwidth]{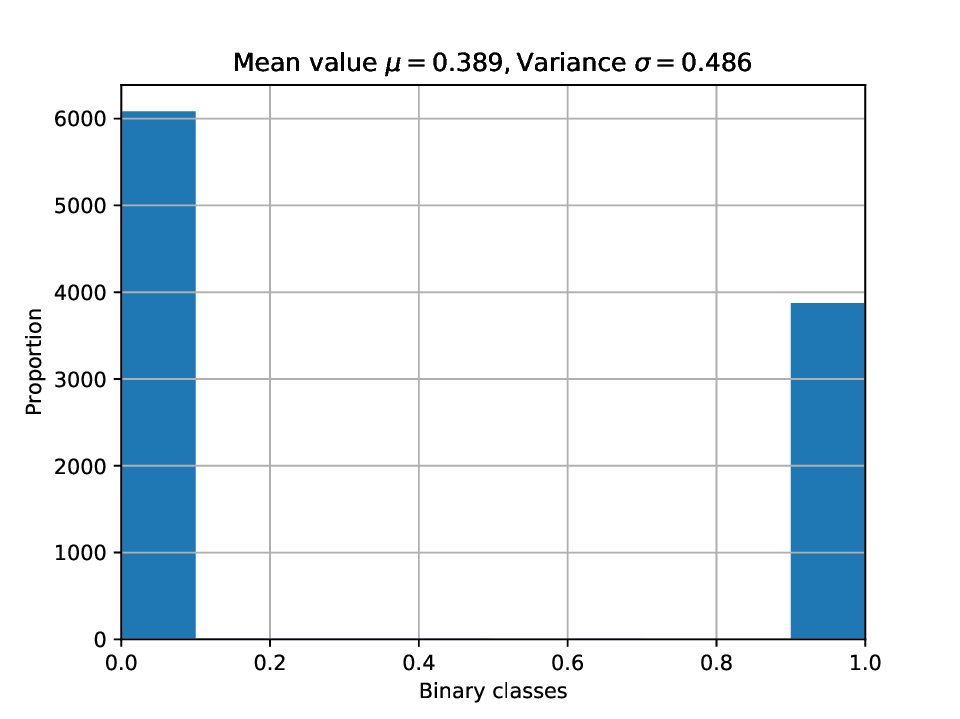}
\includegraphics[width=0.244\textwidth]{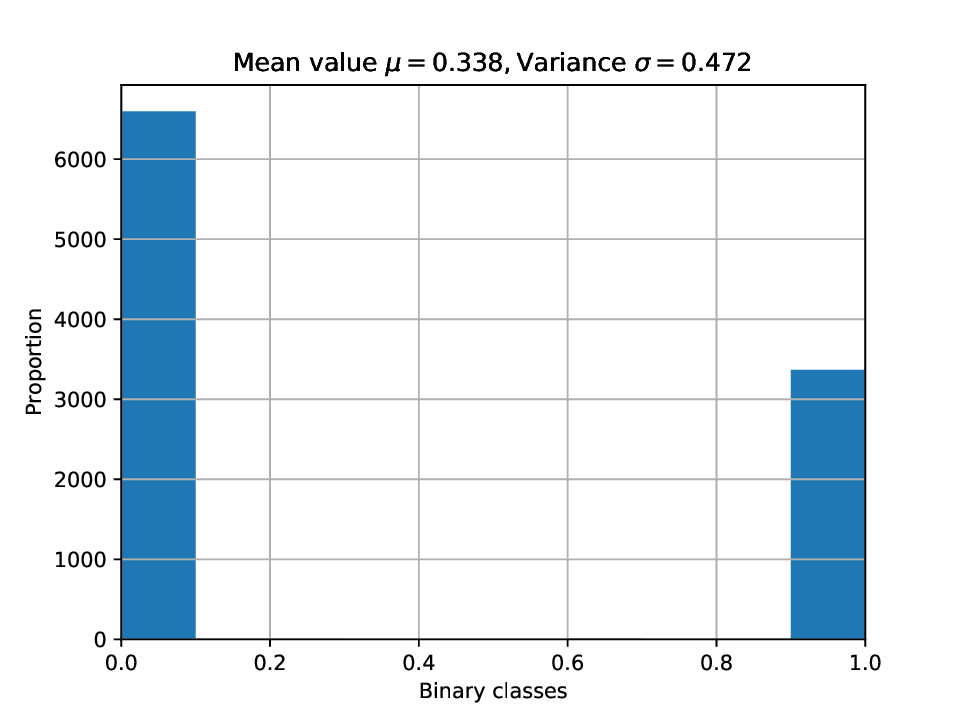}
\includegraphics[width=0.244\textwidth]{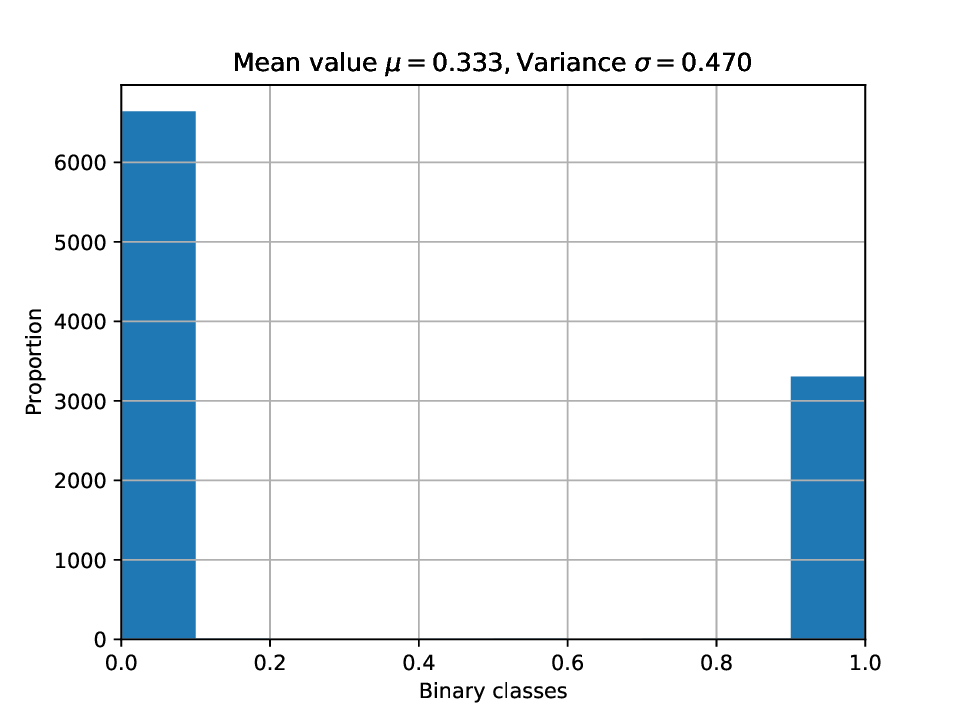}
\caption{Histograms for the degenerate states classifier for 10000 points respectively SLOCC equivalent to the  states $\ket{\Phi_1}$, $\ket{\Phi_2}$, $\ket{\Phi_3}$, and $\ket{\Phi_4}$, from left to right. Classes `0'  and `1' respectively refer to degenerate and non-degenerate states. Here all states are predicted to be degenerate.} \label{fig:5qubits_phis}
\end{figure}

\begin{figure}[!h]
\includegraphics[width=0.45\textwidth]{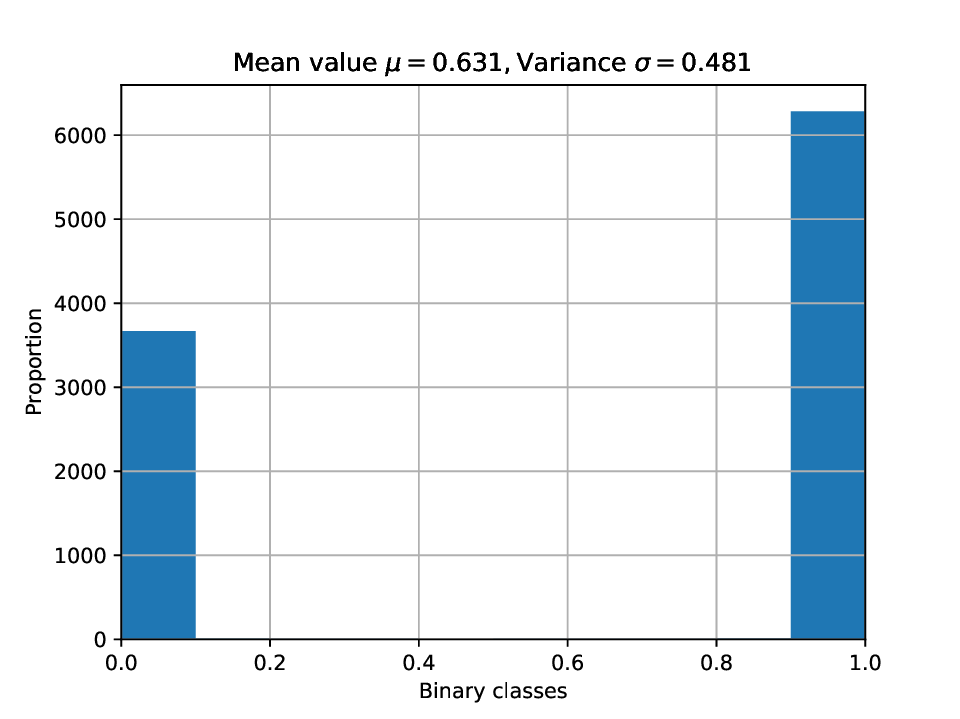}
\includegraphics[width=0.45\textwidth]{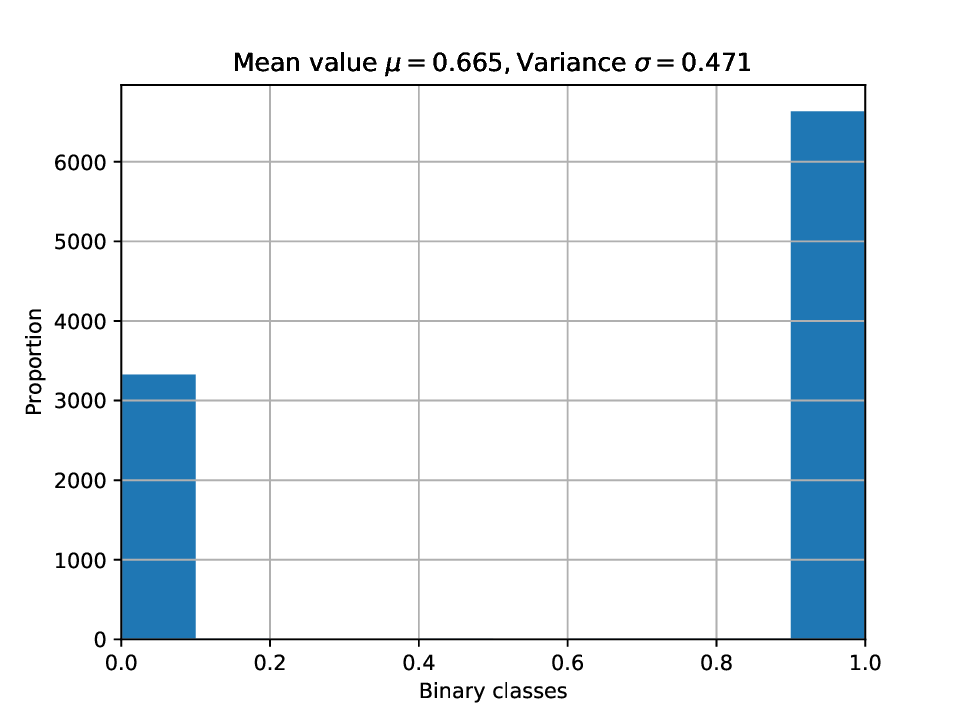}
\caption{Histograms for the degenerate states classifier on 10000 points respectively SLOCC equivalent to $\ket{\delta_1}$ (left) and $\ket{\delta_2}$ (right). Classes `0'  and `1' respectively refer to degenerate and non-degenerate states.  Here both states are predicted to be non-degenerate.} \label{fig:5qubits_delta}
\end{figure}

\end{document}